\documentclass{article}
\usepackage{ijcai22}

\usepackage{times}
\usepackage{soul}
\usepackage{url}
\usepackage[hidelinks]{hyperref}
\usepackage[utf8]{inputenc}
\usepackage[small]{caption}
\usepackage{graphicx}
\usepackage{amsmath}
\usepackage{amsfonts}
\usepackage{amssymb}
\usepackage{amsthm}
\usepackage{booktabs}
\usepackage{bm}
\usepackage{subfigure}
\usepackage{makecell}
\usepackage{algorithm}
\usepackage{algorithmicx}
\usepackage{algpseudocode}
\newtheorem{theorem}{Theorem}

\newtheorem{lemma}[theorem]{Lemma}

\theoremstyle{definition}

\title{Anomaly Detection by Leveraging Incomplete Anomalous Knowledge with Anomaly-Aware Bidirectional GANs}

\author{
Bowen Tian$^1$\and
Qinliang Su$^{1,2}$\footnote{Corresponding author.}\and
Jian Yin$^{2,3}$\and
\affiliations
$^1$School of Computer Science and Engineering, Sun Yat-sen University, Guangzhou, China\\
$^2$Guangdong Key Laboratory of Big Data Analysis and Processing, Guangzhou, China\\
$^3$School of Artificial Intelligence, Sun Yat-sen University, Guangdong, China\\
\emails
tianbw@mail2.sysu.edu.cn, 
\{suqliang,issjyin\}@mail.sysu.edu.cn
}

\begin{document}

\maketitle

\begin{abstract}
  The goal of anomaly detection is to identify anomalous samples from normal ones. In this paper, a small number of anomalies are assumed to be available at the training stage, but they are assumed to be collected only from several anomaly types, leaving the majority of anomaly types not represented in the collected anomaly dataset at all. To effectively leverage this kind of incomplete anomalous knowledge represented by the collected anomalies, we propose to learn a probability distribution that can not only model the normal samples, but also guarantee to assign low density values for the collected anomalies. To this end, an anomaly-aware generative adversarial network (GAN) is developed, which, in addition to modeling the normal samples as most GANs do, is able to explicitly avoid assigning probabilities for collected anomalous samples. Moreover, to facilitate the computation of anomaly detection criteria like reconstruction error, the proposed anomaly-aware GAN is designed to be bidirectional, attaching an encoder for the generator.  Extensive experimental results demonstrate that our proposed method is able to effectively make use of the incomplete anomalous information, leading to significant performance gains compared to existing methods\footnote{Code is available at \url{https://github.com/tbw162/AA-BiGAN}.}.

\end{abstract}

\section{Introduction}

Anomaly detection aims to identify anomalous samples from normal ones, with applications widely found in fields ranging from network security \cite{garcia2009anomaly}, financial fraud detection \cite{abdallah2016fraud}, industrial damage detection to medical diagnosis \cite{Medical} etc. In anomaly detection, normal data generally refers to samples preserving some kinds of regularities (typically composed of one or several types of samples), while anomaly often lacks an explicit and clear definition. Generally, any samples that deviate significantly from the normal ones are considered as anomalies. Obviously, according to this  definition, the types of anomalies are numerous and sometimes even infinite. The extreme diversity of anomalies distinguish the anomaly detection problem from other tasks and also poses a significant challenge to it.

\begin{figure}[!t]
\centering

\subfigure[Unsupervised]{\label{Un_main}\includegraphics[width=0.30\linewidth]{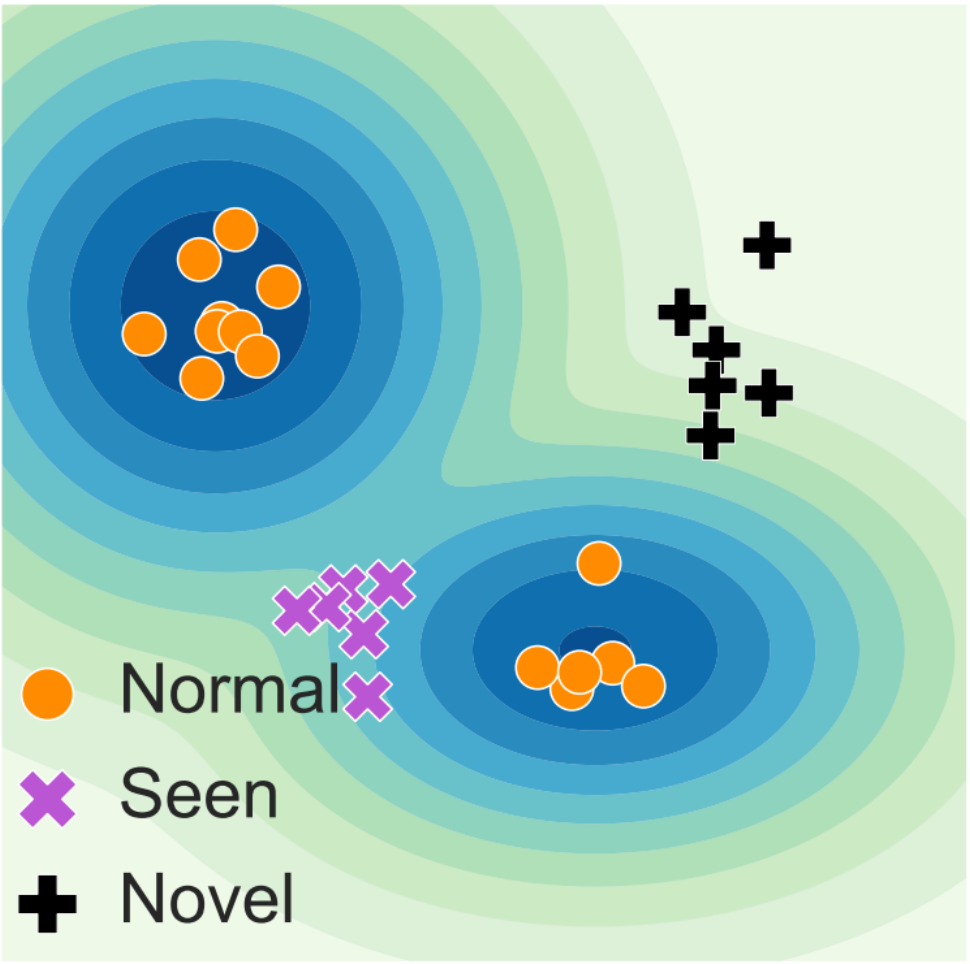}}
\hspace{0.02\linewidth}
\subfigure[Supervised]{\label{Su_main}\includegraphics[width=0.30\linewidth]{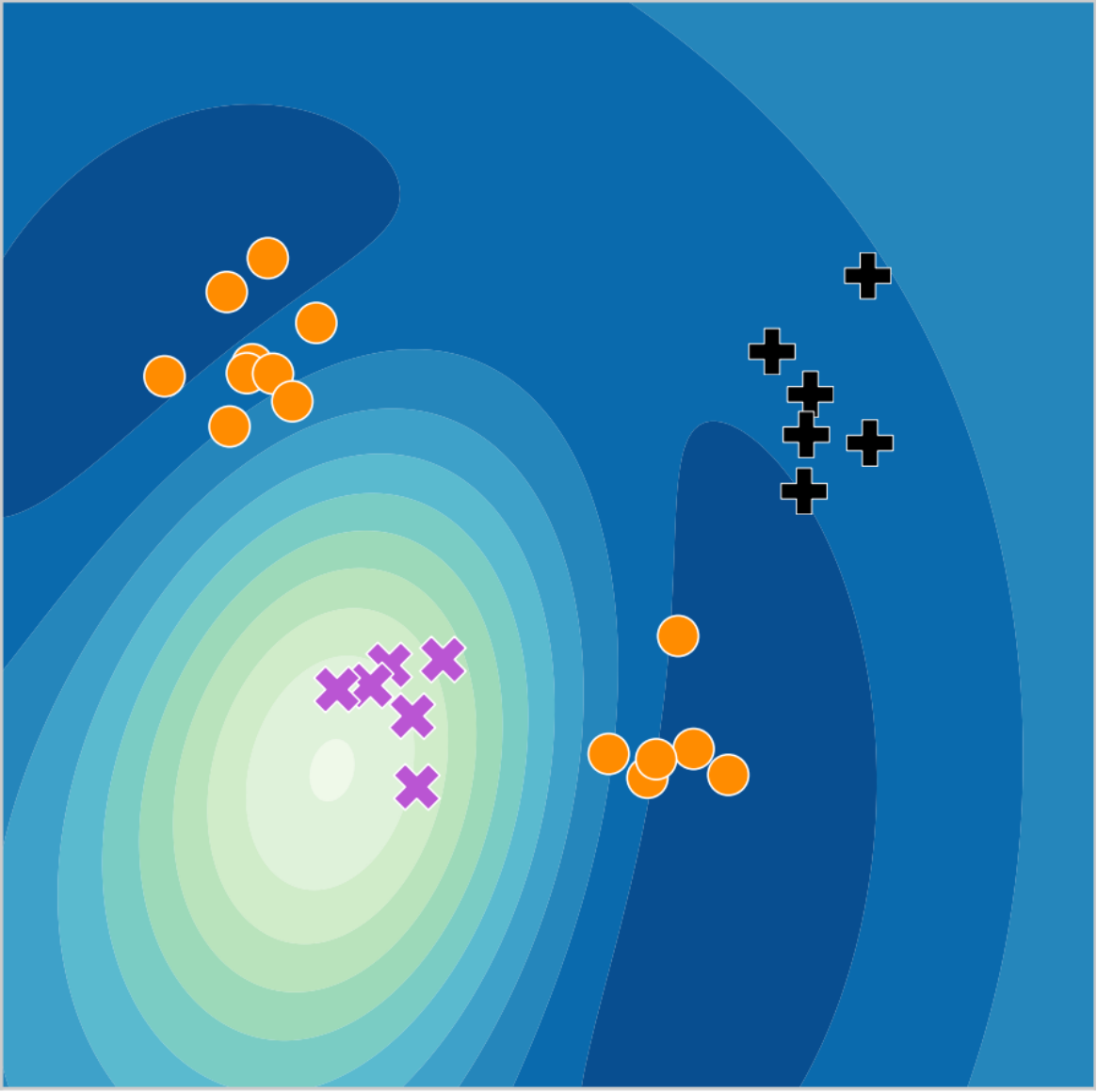}}
\hspace{0.02\linewidth}
\subfigure[Ours]{\label{AEBGAN_main}\includegraphics[width=0.30\linewidth]{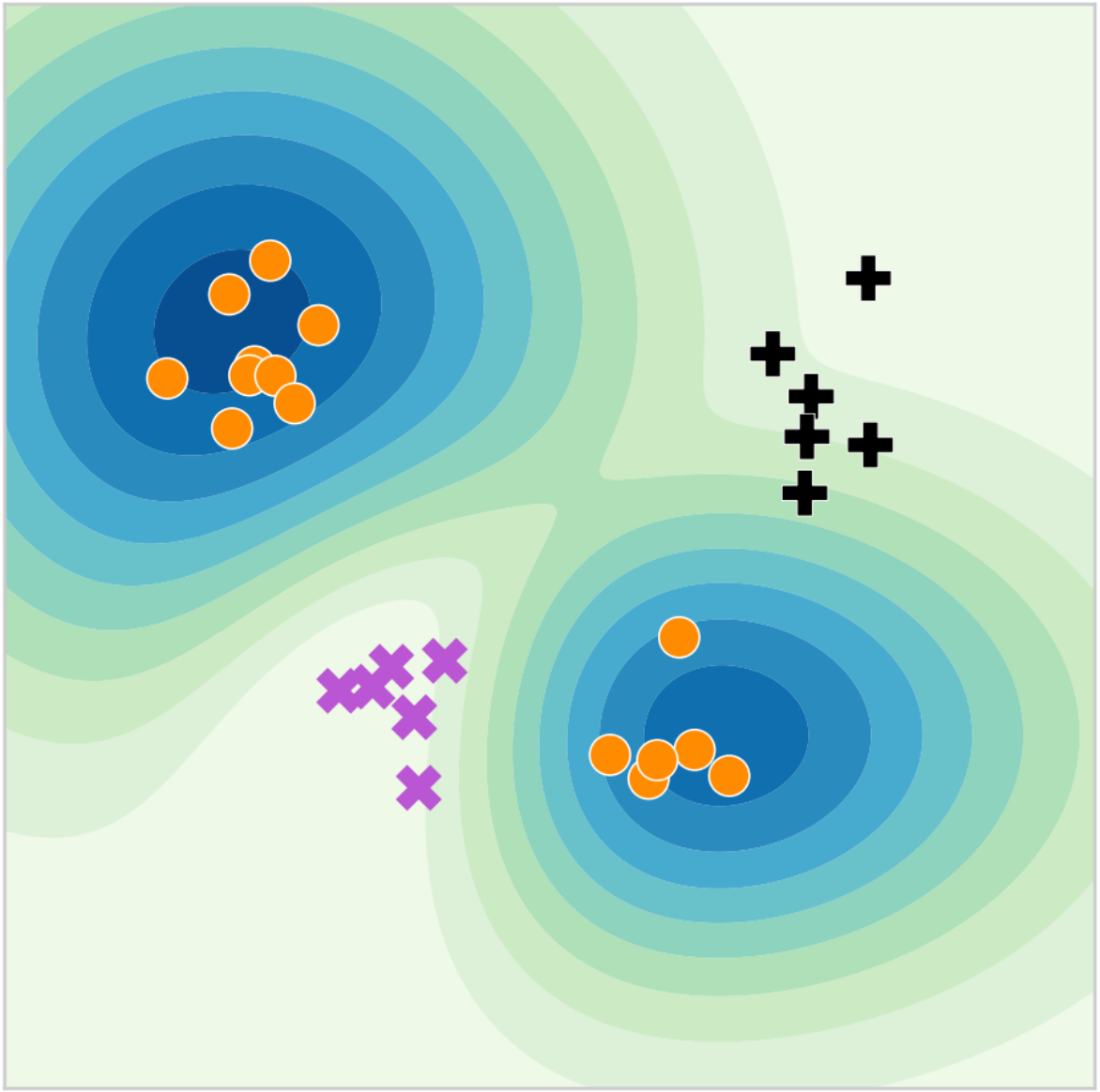}}
\caption{Illustration of different approaches to leverage the collected anomalies for anomaly detection. a) Neglecting the collected anomalies; b) Finding a decision boundary separating the collected anomalies and normal samples; c) Modeling  distribution of normal samples, while ensuring low densities for collected anomalies. The darker the color is, the higher the density value is.}
\label{main_concept}
\vspace{-4mm}
\end{figure}

Existing anomaly detection methods can be roughly grouped into supervised and unsupervised categories. By assuming a full accessibility to anomalies during the training, supervised methods turn the anomaly detection problem into a classification problem. Anomalous samples are generally more difficult and expensive to collect, thus unsupervised methods that only rely on the use of normal samples are used more widely in practice. Despite it is difficult to collect anomalies from every anomaly type, collecting anomalies from one or several types is often possible. For instance, by viewing images of skin diseases as anomalies, it is easy to collect some images of commonly observed skin diseases, but impossible to collect images for all known and unknown skin diseases. Under the circumstances with incompletely collected anomalies, we can still resort to the unsupervised methods by neglecting the anomalies already collected. But this approach will definitely not lead to the optimal performance for not leveraging the available anomalous knowledge, as illustrated in Fig. \ref{main_concept}(a). Supervised methods, on the other hand, essentially seek to find a decision boundary that can separate the normal and collected anomalous samples. However, since the collected anomalies only represent a fraction of all anomaly types, when a anomaly of unseen type is presented, the classifier may not be able to classify it correctly, as illustrated in Fig. \ref{main_concept}(b). 

To leverage the incompletely collected anomalies, a semi-supervised anomaly detection method SSAD is proposed in \cite{SSAD}, which learns a compact hyper-sphere that encompasses the normal samples, while excluding the collected anomalies outside of it. In Deep SAD \cite{DeepSAD}, a deep neural network is trained to encourage the learned representations of normal samples gathering towards a common center, while those of anomalies moving away from it, increasing the discrimination between normal and anomalous samples' representations. Obviously, both methods are established on a good distance metric, which, however, is often very difficult to be found, especially in the high-dimensional data space like images. Recently, \cite{aad2020} proposed to incorporate the information of anomalies identified by an expert into a basic anomaly detector through active learning framework. However, the active learning approach requires the collected anomalies to be processed one by one, losing the merits of batch processing. Moreover, to integrate with the active learning, the basic anomaly detector used in \cite{aad2020} is relatively simple, making it not suitable to be used in high-dimensional data.

In this paper, we argue that the most natural way to address this problem is to approach it from the perspective of probabilistic distribution learning. Under this view, the problem can be reformulated as learning a distribution that can not only model the normal samples well, but also ensure low density values for the collected anomalies, as illustrated in Fig. \ref{main_concept}(c). Inspired by recent successes of deep generative models \cite{GAN2014}, we propose to ground this problem on them. But different from previous generative models, we developed a generative model that can not only capture the distribution of samples from the normal dataset, but also avoid to assign probabilities for samples in the abnormal dataset. Specifically, an anomaly-aware generative adversarial network (GAN) is developed, which is able to explicitly avoid generating samples that look like the training anomalies, apart from the basic capability of generating normal samples. Moreover, due to the high computational cost of directly evaluating the density value, following the unsupervised anomaly detection methods based on GANs \cite{ALAD18}, we make the proposed anomaly-aware GAN to be bidirectional, too. With the bidirectional structure, surrogate metrics ({\it e.g.}, reconstruction error) can be easily calculated to assess the abnormality of a new sample. Extensive experiments were conducted to evaluate the performance of the proposed method. The results demonstrate that the proposed method is able to exploit the collected anomalies to boost the detection performance effectively, and outperforms the comparable baselines by a remarkable margin.

\section{Preliminaries of GAN-Based Unsupervised Anomaly Detection}
One of the main approach for unsupervised anomaly detection is to estimate the probability distribution of normal data. Generative adversarial networks (GANs) \cite{GAN2014}, known for their superior capability of modeling the distribution of high-dimensional data like images, have been applied to detect anomalies. In \cite{AnoGAN17}, a vanilla GAN is trained to model the normal data by playing a min-max game
\begin{align}
	\min_G \max_D V(D, G) \!=  & {\mathbb{E}}_{{\mathbf{x}}\sim p_{data}({\mathbf{x}})} \left[\log D({\mathbf{x}})\right] \nonumber \\
	& + {\mathbb{E}}_{{\mathbf{z}}\sim p({\mathbf{z}})} \left[\log \left(1- D(G({\mathbf{z}}))\right)\right],
\end{align}
where $G(\cdot)$ and $D(\cdot)$ represent the generator and discriminator, respectively; $p_{data}({\mathbf{x}})$ and $p({\mathbf{z}})$ are the distribution of normal data and standard Gaussian distribution, respectively. After training, a joint distribution $p({\mathbf{x}}, {\mathbf{z}})$ over the data and latent code is obtained. However, due to the prohibitive computational complexity involved in the integration $p({\mathbf{x}}) = \int{p({\mathbf{x}}, {\mathbf{z}}) d{\mathbf{z}}}$, the density $p({\mathbf{x}})$ cannot be used to detect anomalies directly. Instead, \cite{AnoGAN17} proposed to use gradient descent methods to find a latent code ${\mathbf{\hat z}}$ that can explain the data ${\mathbf{x}}$ in the latent space, and then use the reconstruction error between original data ${\mathbf{x}}$ and the recovered data $G({\mathbf{\hat z}})$ to assess the abnormality, that is, $\left\|{\mathbf{x}} - G({\mathbf{\hat z}})\right\|^2$. Some other surrogate criteria are also proposed, {\it e.g.}, leveraging the learned discriminator \cite{EGBAD2017,NDA21}. Since the discriminator is not trained to distinguish anomalies from normal data, these criteria generally do not perform as well as the reconstruction error.

\paragraph{Bidirectional-GAN-Based Methods} To reduce the computational cost of searching the latent code, bidirectional GANs \cite{ALI16,ALICE17,BiGAN17} were later used for anomaly detection, in which the associated encoder can output the latent codes directly. For a bidirectional GAN, it includes two joint distributions, the encoder-induced and generator-induced joint distribution
\begin{align}
	p_E({\mathbf{x}}, {\mathbf{z}}) &= p_E({\mathbf{z}}|{\mathbf{x}})p_{data}({\mathbf{x}}), \\
	p_G({\mathbf{x}}, {\mathbf{z}}) &= p_G({\mathbf{x}}|{\mathbf{z}})p({\mathbf{z}}),
\end{align}
where $p_E({\mathbf{z}}|{\mathbf{x}})$ and $p_G({\mathbf{x}}|{\mathbf{z}})$ represent the encoder $E(\cdot)$ and generator $G(\cdot)$, respectively, both of which are parameterized by neural networks. It is proved in \cite{ALI16} that $p_E({\mathbf{x}}, {\mathbf{z}})$ and $p_G({\mathbf{x}}, {\mathbf{z}})$ will converge to the same distribution by playing the following min-max game
\begin{align} \label{obj_BiGAN}
	\min_{G, E} \max_D V\!(D, \!G, \!E)\!\! &=  {\mathbb{E}}_{({\mathbf{x}}, {\mathbf{z}}) \sim p_{E}({\mathbf{ x}}, {\mathbf{ z}})} \left[\log D({\mathbf{x}}, {\mathbf{z}})\right] \nonumber \\
	& \!\!+\! {\mathbb{E}}_{({\mathbf{ x}}, {\mathbf{ z}})\sim p_G({\mathbf{x}}, {\mathbf{z}})} \!\! \left[\log \! \left(1 \!\!- \!\! D({\mathbf{x}}, {\mathbf{z}})\right)\right].
\end{align}
Since $p_E({\mathbf{x}}, {\mathbf{z}})$ is defined as $p_E({\mathbf{x}}, {\mathbf{z}}) = p_E({\mathbf{z}}|{\mathbf{x}})p_{data}({\mathbf{x}})$ and $p_G({\mathbf{x}}, {\mathbf{z}}) = p_E({\mathbf{x}}, {\mathbf{z}})$ holds after convergence, it can be easily seen that the marginal distribution of $p_G({\mathbf{x}}, {\mathbf{z}})$ {\it w.r.t.} ${\mathbf{x}}$ must be $p_{data}({\mathbf{x}})$, and  $p_E({\mathbf{z}}|{\mathbf{x}})$ can be viewed as the inference network of $p_G({\mathbf{x}}, {\mathbf{z}})$. Thus, for a sample ${\mathbf{x}}$,  we can use the inference network $p_E({\mathbf{z}}|{\mathbf{x}})$ to output its latent code ${\mathbf{\hat z}}$ and then use ${\mathbf{\hat z}}$ to compute the reconstruction error.

\section{Anomaly Detection with Anomaly-Aware Bidirectional GANs}

\subsection{Problem Description}

To describe the problem clearly, we define two datasets
\begin{align}
	{\mathcal{X^+}} &\triangleq \{{\mathbf{x}}_1^+, {\mathbf{x}}_2^+, \cdots, {\mathbf{x}}_n^+\},\\
	{\mathcal{X}}^- &\triangleq \{{\mathbf{x}}_1^-, {\mathbf{x}}_2^-, \cdots, {\mathbf{x}}_m^-\},
\end{align}
where ${\mathcal{X}}^+$ and ${\mathcal{X}}^-$ denote the set of normal samples and collected anomalies, respectively. According to the characteristics of anomaly detection tasks, we assume that the normal population can be sufficiently represented by the dataset ${\mathcal{X^+}}$, while because of the extreme diversity of anomalies,  ${\mathcal{X}}^-$ only contain a fraction of anomaly types and cannot be used to represent the whole abnormal population. The problem interested in this paper is to judge a testing sample ${\mathbf{x}}$ is anomalous or not based on the two given datasets ${\mathcal{X}}^+$ and ${\mathcal{X}}^-$. Unsupervised detection methods only leverage the normal dataset ${\mathcal{X}}^+$, while supervised ones leverage both but assume that ${\mathcal{X}}^-$ can represent the entire anomaly population. Both of the methods are not suitable for the considered circumstance. Note that the dataset ${\mathcal{X^+}}$ is assumed to be only composed of normal samples in the analyses below, but we will show experimentally that it could be mixed with some anomalies, leading to only a slight performance drop.

\subsection{Enabling Anomaly-Awareness for Bidirectional GANs under Disjoint Supports}
Deep generative models have been successfully applied to detect anomalies by learning the distribution of normal samples in ${\mathcal{X}}^+$, but none of them make use of the collected anomalies in ${\mathcal{X}}^-$. In this section, to leverage the available anomalous information given by in ${\mathcal{X}}^-$, an \textbf{A}nomaly-\textbf{A}ware \textbf{Bi}directional \textbf{GAN} (AA-BiGAN) is developed, which can explicitly avoid to assign probabilities for samples in ${\mathcal{X}}^-$. To this end, we first transform the traditional bidirectional GAN described in \eqref{obj_BiGAN} into the framework of least-square GAN (LSGAN) \cite{LSGAN17}, which using least-square loss to realize the updating equations. It is known that the optimization objective \eqref{obj_BiGAN} aims to drive the output of discriminator $D(\cdot)$ towards 1 and 0 for samples from $p_E({\mathbf{x}}, {\mathbf{z}})$ and $p_G({\mathbf{x}}, {\mathbf{z}})$, respectively, while encouraging the generator $p_G({\mathbf{x}}|{\mathbf{z}})$ and encoder $p_E({\mathbf{z}}|{\mathbf{x}})$ to confuse the discriminator by 
driving it to output 0.5. Under the LSGAN framework, least-square loss is used to replace the cross-entropy loss in \eqref{obj_BiGAN}, leading to the following updating rules
\begin{align} \label{Discrim_BiLSGAN}
	\min_D V(D) & = {\mathbb{E}}_{({\mathbf{x}}, {\mathbf{z}})\sim p_E({\mathbf{x}}, {\mathbf{z}})}\left[(D({\mathbf{x}}, {\mathbf{z}}) - 1)^2\right] \nonumber \\
	&\quad +  {\mathbb{E}}_{({\mathbf{x}}, {\mathbf{z}}) \sim p_G({\mathbf{x}}, {\mathbf{z}})}\left[(D({\mathbf{x}}, {\mathbf{z}}) - 0)^2\right],
\end{align}
\begin{align} \label{GE_BiLSGAN}
	\min_{G, E} V(G, E) & = {\mathbb{E}}_{({\mathbf{x}}, {\mathbf{z}})\sim p_E({\mathbf{x}}, {\mathbf{z}})}\left[\left(D({\mathbf{x}}, {\mathbf{z}}) - 0.5 \right)^2\right] \nonumber \\
	&\quad +  {\mathbb{E}}_{({\mathbf{x}}, {\mathbf{z}}) \sim p_G({\mathbf{x}}, {\mathbf{z}})}\!\! \left[\left(D({\mathbf{x}}, {\mathbf{z}}) - 0.5 \right)^2\right],
\end{align}
where the generator $G$ and encoder $E$ refer to the conditional distribution $p({\mathbf{x}}|{\mathbf{z}})$ and $p({\mathbf{z}}|{\mathbf{x}})$, respectively; and $p_G({\mathbf{x}}, {\mathbf{z}})=p_G({\mathbf{x}}|{\mathbf{z}})p({\mathbf{z}})$ and $p_E({\mathbf{x}}, {\mathbf{z}})= p_E({\mathbf{z}}|{\mathbf{x}})p_{data}({\mathbf{x}})$. Following similar proofs in LSGAN, we can obtain the lemma.

\begin{lemma}
	\label{lemma:1}
	If the discriminator $D$, generator $G$ and encoder $E$ are updated according to \eqref{Discrim_BiLSGAN} and \eqref{GE_BiLSGAN}, after convergence, the joint distributions $p_G({\mathbf{x}}, {\mathbf{z}})$ and $p_E({\mathbf{x}}, {\mathbf{z}})$ will be equal, that is, $p_G({\mathbf{x}}, {\mathbf{z}}) = p_E({\mathbf{x}}, {\mathbf{z}})$.
\end{lemma}
\begin{proof}
	Please refer to the Supplementary Materials.
\end{proof}
\vspace{-1.0mm}
Although it is known that the marginal of joint distribution $p_G({\mathbf{x}}, {\mathbf{z}})$ is equal to the training data distribution $p_{data}({\mathbf{x}})$ after convergence, the method still lacks the ability of explicitly avoiding to assign probabilities for samples in ${\mathcal{X}}^-$. To achieve this goal, we modify the updating rules in \eqref{Discrim_BiLSGAN} and \eqref{GE_BiLSGAN} by adding an anomaly-relevant term
\begin{align} \label{discrim_anoAwareGAN}
	\min_D V(D) & = {\mathbb{E}}_{({\mathbf{x}}, {\mathbf{z}})\sim p_E^+({\mathbf{x}}, {\mathbf{z}})}\left[(D({\mathbf{x}}, {\mathbf{z}}) - 1)^2\right] \nonumber \\
	&\quad + {\mathbb{E}}_{({\mathbf{x}}, {\mathbf{z}}) \sim p_E^-({\mathbf{x}}, {\mathbf{z}})}\left[(D({\mathbf{x}}, {\mathbf{z}}) - 0)^2\right] \nonumber \\
	&\quad +  {\mathbb{E}}_{({\mathbf{x}}, {\mathbf{z}}) \sim p_G({\mathbf{x}}, {\mathbf{z}})}\left[(D({\mathbf{x}}, {\mathbf{z}}) - 0)^2\right],
\end{align}
\begin{align} \label{Gen_anoAwareGAN}
	\min_{G, E} V(G, E) & = {\mathbb{E}}_{({\mathbf{x}}, {\mathbf{z}})\sim p_E^+({\mathbf{x}}, {\mathbf{z}})}\left[\left(D({\mathbf{x}}, {\mathbf{z}}) - 0.5 \right)^2\right] \nonumber \\
	&\quad + \! {\mathbb{E}}_{({\mathbf{x}}, {\mathbf{z}})\sim p_E^-({\mathbf{x}}, {\mathbf{z}})}\!\! \left[\left(D({\mathbf{x}}, {\mathbf{z}}) - 0.5 \right)^2\right] \nonumber \\
	&\quad +\!  {\mathbb{E}}_{({\mathbf{x}}, {\mathbf{z}}) \sim p_G({\mathbf{x}}, {\mathbf{z}})}\!\! \left[\left(D({\mathbf{x}}, {\mathbf{z}}) \!- \! 0.5 \right)^2\right],
\end{align}
where the distributions 
\begin{align}
	p_E^+({\mathbf{x}}, {\mathbf{z}}) &\triangleq p_E({\mathbf{z}}|{\mathbf{x}})p_{{\mathcal{X}}}^+({\mathbf{x}}), \\
	 p_E^-({\mathbf{x}}, {\mathbf{z}}) &\triangleq p_E({\mathbf{z}}|{\mathbf{x}})p_{{\mathcal{X}}}^-({\mathbf{x}}); 
\end{align}
$p_{\mathcal{X}}^+({\mathbf{x}})$ and $p_{\mathcal{X}}^-({\mathbf{x}})$ represent the distributions of samples in ${\mathcal{X}}^+$ and ${\mathcal{X}}^-$, respectively.
Although the exact expressions of $p_{\mathcal{X}}^+({\mathbf{x}})$ and $p_{\mathcal{X}}^-({\mathbf{x}})$ are not known, we can easily draw samples from them. The optimal discriminator is obtained
\begin{equation} \label{optimaldiscrim_anoAwareGAN}
	D^*({\mathbf{x}}, {\mathbf{z}}) = \frac{p_E^+({\mathbf{x}}, {\mathbf{z}})}{p_E^+({\mathbf{x}}, {\mathbf{z}})+p_E^-({\mathbf{x}}, {\mathbf{z}})+p_G({\mathbf{x}}, {\mathbf{z}})}.
\end{equation}
By substituting the optimal discriminator $D^*({\mathbf{x}}, {\mathbf{z}})$ into \eqref{Gen_anoAwareGAN}, it can be easily verified that under the prerequisite of disjoint support  $Supp(p_{\mathcal{X}}^+({\mathbf{x}})) \cap Supp(p_{\mathcal{X}}^-({\mathbf{x}}))=\emptyset$, the optima is achieved if and only if $p_E^+({\mathbf{x}}, {\mathbf{z}})=p_G({\mathbf{x}}, {\mathbf{z}})$. Therefore, we have the following theorem
\begin{theorem}
	\label{lemma_disjoint}
	Suppose $Supp(p_{\mathcal{X}}^+({\mathbf{x}}) \cap Supp(p_{\mathcal{X}}^-({\mathbf{x}}))=\emptyset$. If the discriminator $D$, generator $G$ and encoder $E$ are updated according to \eqref{discrim_anoAwareGAN} and \eqref{Gen_anoAwareGAN}, after convergence, the two distributions $p_G({\mathbf{x}}, {\mathbf{z}})$ and $p_E^+({\mathbf{x}}, {\mathbf{z}})$ will be equal, that is, $p_G({\mathbf{x}}, {\mathbf{z}}) = p_E^+({\mathbf{x}}, {\mathbf{z}})$.
\end{theorem}
\begin{proof}
	Please refer to Supplementary Materials.
\end{proof}
\vspace{-1.0mm}
From Theorem \ref{lemma_disjoint}, it can be seen that due to $p_E^+({\mathbf{x}}, {\mathbf{z}}) \triangleq p({\mathbf{z}}|{\mathbf{x}})p_{{\mathcal{X}}}^+({\mathbf{x}})$, the marginal of generative distribution $p_G({\mathbf{x}}, {\mathbf{z}})$ will converge to $p_{\mathcal{X}}^+({\mathbf{x}})$, which is the same as previous GANs when  $p_{data}({\mathbf{x}})$ is set as $p_{\mathcal{X}}^+({\mathbf{x}})$. Although the new updating rules do not change the converged distribution, they bring us a different optimal discriminator as shown in \eqref{optimaldiscrim_anoAwareGAN}, which is aware of the distribution of collected anomalies. To see this, let us examine the training dynamics of the generator. Suppose that a sample $({\mathbf{\tilde x}}, {\mathbf{\tilde z}})$ is generated from the generative model $p_G({\mathbf{x}}, {\mathbf{z}})$. If the sample $({\mathbf{\tilde x}}, {\mathbf{\tilde z}})$ looks similar to the collected anomalies, that is, $p_{\mathcal{X}}^-({\mathbf{\tilde x}}, {\mathbf{\tilde z}})\gg p_{\mathcal{X}}^+({\mathbf{\tilde x}}, {\mathbf{\tilde z}})$, according to the discriminator in \eqref{optimaldiscrim_anoAwareGAN}, its output value must be very small {\it i.e.}, close to $0$. Since the target value of discriminator is 0.5 for the training of generator, the significant discrepancy between the two values will push the generator quickly moving away from current state. Therefore, the discriminator \eqref{optimaldiscrim_anoAwareGAN} is able to prevent the generator from generating anomaly-like samples. By contrast, it can be easily shown that when setting $p_{data}({\mathbf{x}})$ as $p_{\mathcal{X}}^+({\mathbf{x}})$, the optimal discriminator of previous GANs obtained by \eqref{obj_BiGAN} or \eqref{Discrim_BiLSGAN} and \eqref{GE_BiLSGAN} is
\begin{equation} \label{optimaldiscrim_AwareGAN}
	\widetilde D^*({\mathbf{x}}, {\mathbf{z}}) = \frac{p_E^+({\mathbf{x}}, {\mathbf{z}})}{p_E^+({\mathbf{x}}, {\mathbf{z}})+p_G({\mathbf{x}}, {\mathbf{z}})},
\end{equation}
which does not include the anomaly-relevant term $p_E^-({\mathbf{x}}, {\mathbf{z}})$ in the discriminator. Without this term, even if a sample that looks similar to anomalies is generated, the discriminator is not guaranteed to output a small value. As a result, the discriminator $\widetilde D(\cdot)$ in \eqref{optimaldiscrim_AwareGAN} does not have the capability of explicitly preventing the generator from generating anomaly-like samples. Although theoretically the marginal of both generative distributions, no matter which discriminator is used, will converge to $p_{\mathcal{X}}^+({\mathbf{x}})$ eventually, this is established on the assumptions of infinite training data and infinite representational capacities of neural networks. However, none of these conditions can be satisfied in practice, making the learned distribution only be an approximate of the ideal distribution $p_{\mathcal{X}}^+({\mathbf{x}})$. Therefore, the collected anomalies are possibly assigned with some non-negligible probabilities in the learned distribution. However, if the anomaly-aware discriminator \eqref{optimaldiscrim_anoAwareGAN} is used, this kind of possibilities could be reduced significantly, thanks to the model's capability of explicitly avoiding to generate anomaly-like samples.

\subsection{Relaxing the Prerequisites}
The derivation of anomaly-aware bidirectional GANs in the previous section hinges on the prerequisite of disjoint supports for distributions of normal and anomalous samples $p_{\mathcal{X}}^+({\mathbf{x}})$ and $p_{\mathcal{X}}^-({\mathbf{x}})$, that is, $Supp(p_{\mathcal{X}}^+({\mathbf{x}})) \cap Supp(p_{\mathcal{X}}^-({\mathbf{x}}))=\emptyset$. However, the disjoint condition may not always hold in practice, especially when the anomalies and normal samples share lots of commonalities. In this section, we show that this prerequisite can be removed by adjusting the target values of discriminator. Moreover, we can also show that the target values do not need to be restricted to some fixed values, but could have lots of feasible configurations. To be specific, we have the following theorem.
\begin{theorem}
	\label{lemma_GenerousanoAwere}
	If the discriminator $D$, generator $G$ and encoder $E$ are updated according to
	\begin{align} \label{discrim_GeneralanoAwareGAN}
		\min_D V(D) & = {\mathbb{E}}_{({\mathbf{x}}, {\mathbf{z}})\sim p_E^+({\mathbf{x}}, {\mathbf{z}})}\left[(D({\mathbf{x}}, {\mathbf{z}}) - a)^2\right] \nonumber \\
		&\quad + \! {\mathbb{E}}_{({\mathbf{x}}, {\mathbf{z}}) \sim p_E^-({\mathbf{x}}, {\mathbf{z}})}\!\! \left[\!\!\left(\!D({\mathbf{x}}, {\mathbf{z}}) \!-\! \frac{a\!+\!b}{2}\!\right)^2\!\right] \nonumber \\
		&\quad +  {\mathbb{E}}_{({\mathbf{x}}, {\mathbf{z}}) \sim p_G({\mathbf{x}}, {\mathbf{z}})}\left[(D({\mathbf{x}}, {\mathbf{z}}) - b)^2\right],
	\end{align}
	\begin{align} \label{Gen_GeneralanoAwareGAN}
		\min_{G, E} V(G, E) & = {\mathbb{E}}_{({\mathbf{x}}, {\mathbf{z}})\sim p_E^+({\mathbf{x}}, {\mathbf{z}})}\left[\left(D({\mathbf{x}}, {\mathbf{z}}) - c \right)^2\right] \nonumber \\
		&\quad + \! {\mathbb{E}}_{({\mathbf{x}}, {\mathbf{z}})\sim p_E^-({\mathbf{x}}, {\mathbf{z}})}\!\! \left[\left(D({\mathbf{x}}, {\mathbf{z}}) - c \right)^2\right] \nonumber \\
		&\quad +\!  {\mathbb{E}}_{({\mathbf{x}}, {\mathbf{z}}) \sim p_G({\mathbf{x}}, {\mathbf{z}})}\!\! \left[\left(D({\mathbf{x}}, {\mathbf{z}}) \!- \! c \right)^2\right],
	\end{align}
the two distributions $p_G({\mathbf{x}}, {\mathbf{z}})$ and $p_E^+({\mathbf{x}}, {\mathbf{z}})$ after convergence will be equal, that is, $p_G({\mathbf{x}}, {\mathbf{z}}) = p_E^+({\mathbf{x}}, {\mathbf{z}})$.
\end{theorem}
\begin{proof}
	Please refer to the Supplementary Materials.
\end{proof}
\vspace{-1.0mm}
The proof of this theorem relies on the use of Jensen inequality and is more complex than that of Theorem \ref{lemma_disjoint}. According to this theorem, we are able to obtain a bidirectional generative model that can explicitly avoid to generate anomaly-like samples even if the disjoint supports condition does not hold. For the parameters $a$, $b$ and $c$, theoretically they could be set arbitrary. In the experiments, we follow LSGAN and set $a$ and $b$ as $1$ and $0$, respectively. For the value of $c$, we find that the performance is not sensitive to its value, as demonstrated by the results in Supplementary Materials. We observed that letting it far away from $a$, $b$ and $\frac{a+b}{2}$ generally deliver a slightly better performance, thus we simply set it set as $\frac{3}{4}$ for all experiments.

\subsection{Detection Methods}

With the anomaly-aware bidirectional GAN, we can apply it to identify anomalies for new samples. Due to the difficulties of computing the marginal densities $p_G({\mathbf{x}}) =\int{p_G({\mathbf{x}}, {\mathbf{z}})d{\mathbf{z}}}$, reconstruction error is used as the surrogate criteria to replace the density value, as found in many generative-model based anomaly detection methods \cite{vaerecon,ALAD18}. Specifically, we employ the encoder $p_E({\mathbf{z}}|{\mathbf{x}})$ to produce a latent code ${\mathbf{z}}$ for the sample ${\mathbf{x}}$, and then use the generator to generate a sample ${\mathbf{\hat x}}$ from the code ${\mathbf{z}}$. The error $\left\|{\mathbf{x}} - {\mathbf{\hat x}}\right\|^2$ is then used to indicate the degree of abnormality of input sample ${\mathbf{x}}$. It is widely observed that the vanilla bidirectional GANs do not reconstruct the input well. Thus, to increase their reconstruction ability, as proposed in \cite{ALICE17}, one more discriminator is added to distinguish the pairs between $({\mathbf{x}}, {\mathbf{x}})$ and $({\mathbf{x}}, {\mathbf{\hat x}})$, with the specific expression deferred to the Supplementary Materials. In addition to the reconstruction error, some papers \cite{GANomaly18} also proposed to use the norm of latent code $\left\|{\mathbf{ z}}\right\|$ to detect anomalies by noticing that the codes of normal samples follow a standard normal distribution ${\mathcal{N}}({\mathbf{0}}, {\mathbf{I}})$. Thus, if a code's norm $\left\|{\mathbf{ z}}\right\|$ is large, it is highly suspected of being an anomaly. In practice, we find that both of the two criteria work well, but show some differences on specific datasets. In our experiments, we simply choose the one with better performance on the validation subset for a specific dataset.

\vspace{-1.0mm}
\section{Related work}
\label{related work}
With the recent advancement of deep learning, neural networks have been used to help anomaly detection \cite{pang2021deep}. Among the existing methods, a typical approach is to train a deep model to reconstruct the normal data and then employ the reconstruction error  to detect anomalies \cite{recon2015}. To achieve better and robust performance, variational auto-encoders (VAE) \cite{VAE2014} are further used in \cite{vaerecon}. To increase the distinction between the reconstruction errors of normal and anomalous samples, a memory module \cite{MAMAE} and discriminators \cite{OCGAN19} are added into the autoencoders. Another type of mainstream generative models, generative adversarial networks (GAN), are also widely used for anomaly detection \cite{AnoGAN17}. To obtain the reconstruction error more cheaply, bidirectional GANs \cite{ALI16} are proposed to use in \cite{ALAD18}. In addition to reconstruction-based methods, many methods established on density estimation are also studied, {\it e.g.}, estimating the density distribution of normal samples using energy-based models \cite{DSEBM16} or deep Gaussian mixture models \cite{DAGMM}, and then using the density value to detect anomalies. There are also many other methods that resort to the one-class classifier, which assumes that normal samples can often be encompassed by a compact hypersphere \cite{SVDD} or separated by a hyperplane  \cite{OCSVM}. Deep SVDD \cite{DeepSVDD} is further proposed, which uses neural networks to learn discriminative representations for normal and anomalous samples and then finds a hypersphere to separate them like SVDD.  In recent years, it has been pointed out that it is possible to collect a handful of anomalies before the training in practice. Early methods simply turn this problem into a binary classification problem \cite{SS-DGM}, without considering the incompleteness of the collected anomalies. Differently, SSAD \cite{SSAD} proposed to ground the problem on unsupervised detection method SVDD, while ensuring the available anomalous samples are outside of the hypersphere. Later, deep SAD \cite{DeepSAD} proposed to learn a mapping function that encourages the representations of normal samples gathering toward a center, while those of collected anomalies moving away from it. Although both methods achieve superior performance, these methods heavily rely on the use of a good distance metric, which is often difficult to be found in high-dimensional data. Recently, \cite{aad2020} seeks to use active learning to incorporate the anomalies identified by an expert into the existing basic anomaly detector. \cite{pang2021toward} proposed a reinforcement-learning-based method to actively seek novel types of anomalies. However, the architecture of basic anomaly detectors employed in both methods is relatively simple, impeding them from being applied to high-dimensional data with complex structure like images.

\section{Experiments}
\subsection{Experimental Setups}

\paragraph{Datasets} Following the setups in Deep SAD \cite{DeepSAD}, we first evaluate the detection performance of our method on three image datasets. i) \emph{MNIST}: A dataset consisting of 60000 training and 10000 testing handwritten digit images of 28$\times$28  from 10 classes. ii) \emph{F-MNIST}: A dataset consisting of 60000 training and 10000 testing fashion images of 28$\times$28 from 10 classes.  iii) \emph{CIFAR10}: A dataset consisting of 50000 training and 10000 testing images from ten classes. Then, six classic anomaly detection datasets are used to further evaluate the performance of proposed method.

\paragraph{Training \& Evaluation} Following the paper \cite{DeepSAD}, for each dataset, we select one category as normal, while treating the remaining nine types as anomalies. To mimic the circumstance of incomplete anomalous information, a proportion of samples from one of the nine anomalous categories have been collected. The proposed model is trained on the normal samples and collected anomalies. The testing dataset is splitted into a validation and testing dataset with a ratio of $20\%$ and $80\%$. The hyperparameters are fine-tuned on the validation dataset, please refer to the Supplementary Material for more details of training. The area under the receiver operating characteristic curve (AUROC) is employed as the performance criteria. The reported experimental results are averaged over 90 experiments ({\it i.e.,} 10 choices of normal categories $\times$ 9 choices of anomalous categories).

\begin{table*}[!htb]
	\centering
    \small
    \setlength\tabcolsep{7pt}
    \renewcommand\arraystretch{0.94}
	{\begin{tabular}{ccccccccc}
			\toprule
			\makecell[c]{\\Data} & \makecell[c]{\\$\gamma_l$} & \makecell[c]{SSAD\\ Raw~ } & \makecell[c]{SSAD\\ Hybrid~} & \makecell[c]{\\SS-DGM} &\makecell[c]{Supervised\\ Classifier} & \makecell[c]{Deep\\ SAD~}  &\makecell[c]{\\NDA}&  \makecell[c]{AA-BiGAN \\ (ours)}\\
			\cmidrule(lr){1-1} \cmidrule(lr){2-2} \cmidrule(lr){3-3} \cmidrule(lr){4-4} \cmidrule(lr){5-5} \cmidrule(lr){6-6} \cmidrule(lr){7-7} \cmidrule(lr){8-8} 
			\cmidrule(lr){9-9} 
			& .00 & 96.0 $\pm$ 2.9& \textbf{96.3} $\pm$ \textbf{2.5} &  n.a. & n.a. & 92.8 $\pm$ 4.9 & 86.5 $\pm$ 9.5  & \textbf{96.3 $\pm$ 2.8}  \\
			& .01 &96.6 $\pm$ 2.4& 96.8 $\pm$ 2.3 & 89.9 $\pm$ 9.2 &92.8 $\pm$ 5.5 & 96.4 $\pm$ 2.7 & 96.5 $\pm$ 3.7  &\textbf{97.4 $\pm$ 2.1} \\
			MNIST & .05 &93.3 $\pm$ 3.6& 97.4 $\pm$ 2.0 & 92.2 $\pm$ 5.6 & 94.5 $\pm$ 4.6  & 96.7 $\pm$ 2.4 & 96.8 $\pm$ 3.2 & \textbf{97.8 $\pm$ 2.0} \\
			& .10 &90.7 $\pm$ 4.4 & 97.6 $\pm$ 1.7 & 91.6 $\pm$ 5.5 & 95.0 $\pm$ 4.7 & 96.9 $\pm$ 2.3 & 96.9 $\pm$ 3.0 & \textbf{98.4 $\pm$ 1.7} \\
			& .20 & 87.2 $\pm$ 5.6  &97.8 $\pm$ 1.5 & 91.2 $\pm$ 5.6 & 95.6 $\pm$ 4.4 & 96.9 $\pm$ 2.4 & 97.1 $\pm$ 2.9 & \textbf{98.2 $\pm$ 1.6} \\
			\cmidrule(lr){1-1} \cmidrule(lr){2-2} \cmidrule(lr){3-3} \cmidrule(lr){4-4} \cmidrule(lr){5-5} \cmidrule(lr){6-6} \cmidrule(lr){7-7} \cmidrule(lr){8-8}  \cmidrule(lr){9-9} 
			& .00 &92.8 $\pm$ 4.7 & 91.2 $\pm$ 4.7 & n.a. & n.a. & 89.2 $\pm$ 6.2 & 82.7 $\pm$ 11.4 & 
			\textbf{93.0 $\pm$ 4.8} \\
			& .01 &92.1 $\pm$ 5.0 &89.4 $\pm$ 6.0 & 65.1 $\pm$ 16.3 & 74.4 $\pm$ 13.6 & 90.0 $\pm$ 6.4 & 90.1 $\pm$ 8.5 & \textbf{94.4 $\pm$ 5.0} \\
			F-MNIST & .05  &88.3 $\pm$ 6.2& 90.5 $\pm$ 5.9 & 71.4 $\pm$ 12.7 & 76.8 $\pm$ 13.2 & 90.5 $\pm$ 6.5 & 91.0 $\pm$ 7.1 & \textbf{94.6 $\pm$ 4.2} \\
			& .10 &85.5 $\pm$ 7.1 &91.0 $\pm$ 5.6 & 72.9 $\pm$ 12.2 & 79.0 $\pm$ 12.3 & 91.3 $\pm$ 6.0 & 91.4 $\pm$ 7.0 & \textbf{94.7} $\pm$ \textbf{4.3} \\
			& .20 & 82.0 $\pm$ 8.0 & 89.7 $\pm$ 6.6 & 74.7 $\pm$ 13.5 & 81.4 $\pm$ 12.0 & 91.0 $\pm$ 5.5  & 91.4 $\pm$ 7.1 & \textbf{94.8 $\pm$ 4.1} \\
			\cmidrule(lr){1-1} \cmidrule(lr){2-2} \cmidrule(lr){3-3} \cmidrule(lr){4-4} \cmidrule(lr){5-5} \cmidrule(lr){6-6} \cmidrule(lr){7-7}\cmidrule(lr){8-8} \cmidrule(lr){9-9}
			& .00 & 62.0 $\pm$ 10.6 &63.8 $\pm$ 9.0 & n.a. & n.a. & 60.9 $\pm$ 9.4 & 64.8 $\pm$ 8.2 & \textbf{65.0 $\pm$ 9.4}   \\
			& .01 & 73.0 $\pm$ 8.0 & 70.5 $\pm$ 8.3 & 49.7 $\pm$ 1.7 & 55.6 $\pm$ 5.0 & 72.6 $\pm$ 7.4 & 73.4 $\pm$ 7.7 & \textbf{80.2 $\pm$ 6.6} \\
			CIFAR-10 & .05 & 71.5 $\pm$ 8.1& 73.3 $\pm$ 8.4 & 50.8 $\pm$ 4.7& 63.5 $\pm$ 8.0 & 77.9 $\pm$ 7.2  &78.9 $\pm$ 7.7 & \textbf{81.5 $\pm$ 6.4} \\
			& .10 &70.1 $\pm$ 8.1 & 74.0 $\pm$ 8.1 & 52.0 $\pm$ 5.5 & 67.7 $\pm$ 9.6 & 79.8 $\pm$ 7.1 & 80.1 $\pm$ 7.6 & \textbf{83.6 $\pm$ 6.4} \\
			& .20 &67.4 $\pm$ 8.8& 74.5 $\pm$ 8.0 & 53.2 $\pm$ 6.7 & 80.5 $\pm$ 5.9 & 81.9 $\pm$ 7.0 & 81.6 $\pm$ 7.3 & \textbf{84.8 $\pm$ 7.2}\\
			\bottomrule
	\end{tabular}}
	\vspace{-2mm}
	\caption{Averaged AUROC under different collected-anomaly ratios $\gamma_l$.}
	\label{tal:1}
	\vspace{-2.5mm}
\end{table*}

\paragraph{Baselines}
Semi-supervised anomaly detection methods are used for comparison: \emph{SSAD} \cite{SSAD}, \emph{SS-DGM} \cite{SS-DGM}, \emph{Deep SAD} \cite{DeepSAD}, negative data augmentation method (\emph{NDA}) \cite{NDA21}, and active anomaly detection (\emph{AAD}) \cite{aad2020} for tabular dataset. In addition, a supervised binary classifier is also trained for comparison. The performance of these methods, except NDA and AAD, are quoted from Deep SAD in \cite{DeepSAD}. The NDA and AAD papers concentrate on different scenarios as ours, thus we report their performance by running and tuning their publicized code on our experiments.

\subsection{Performance and Analysis}
\paragraph{Overall Performance} 
Table \ref{tal:1} shows the performance of our proposed model. To investigate how the number of collected anomalies affects the performance, the performance is evaluated under different values of $\gamma_l$, where $\gamma_l$ is defined as the ratio between the number of collected anomalies and the number of normal samples. When $\gamma_l$ is set to $0$, the semi-supervised methods are reduced to the unsupervised ones. It can be seen that as long as a very small fraction of anomalies is used, a significant performance improvement can be observed, especially on the complex dataset of CIFAR-10. And as the ratio of available anomalies further increases, a steady improvement can be observed. Under the specific case of $\gamma_l=0.01$, the performance gains observed in our proposed method over the best baseline on MNIST, F-MNIST and CIFAR-10 are $0.6\%$, $2.3\%$ and $6.8\%$, respectively,  demonstrating that our method can exploit the partially-observed anomalies effectively to boost the overall detection performance. Moreover, it can be also seen that as the dataset of interest becomes more complex, the advantages of our proposed method become more clear. This may be attributed to the strong ability of GANs in modeling the probabilistic distribution of image data. By contrast, due to the difficulties of finding an appropriate metric to measure the distance in high-dimensional data space, distance-based detection methods, like SSAD and Deep SAD, struggle in such scenarios. As for the NDA, although it is also established on GANs, it focuses more on the generation of images, lacking the necessary bidirectional structure to produce the surrogate detection criteria like reconstruction error as in our model.

\paragraph{Performance with polluted normal dataset ${\mathcal{X}}^+$} In some scenarios, it may be difficult to obtain a clean dataset ${\mathcal{X}}^+$, in which a small proportion of anomalous samples may be mixed. To investigate how our proposed model performs under such circumstances, experiments are conducted under different ratios of pollution $\gamma_p$, where $\gamma_p$ is defined as the ratio between the number of pollution anomalies and the number of normal samples in the training. Fig. \ref{img:resultp} shows the detection performance of different models as a function of $\gamma_p$ on different datasets, in which the $\gamma_l$ is fixed as $0.05$. It can be seen that the performance of almost all methods deteriorates as the pollution ratio $\gamma_p$ increases from $0.0$ to $0.2$. However, the proposed method decreases much slower than the compared ones. This may be because the probabilistic method is more tolerant to data pollution than distance-based methods. We can observe that the performance of NDA decreases quickly as $\gamma_p$ increases. This may be because data pollution breaks its required distribution disjoint assumption in NDA.

\begin{figure}[!tb]
    \centering
    \begin{minipage}{0.46\linewidth}
        \centering
        \includegraphics[width=1\linewidth]{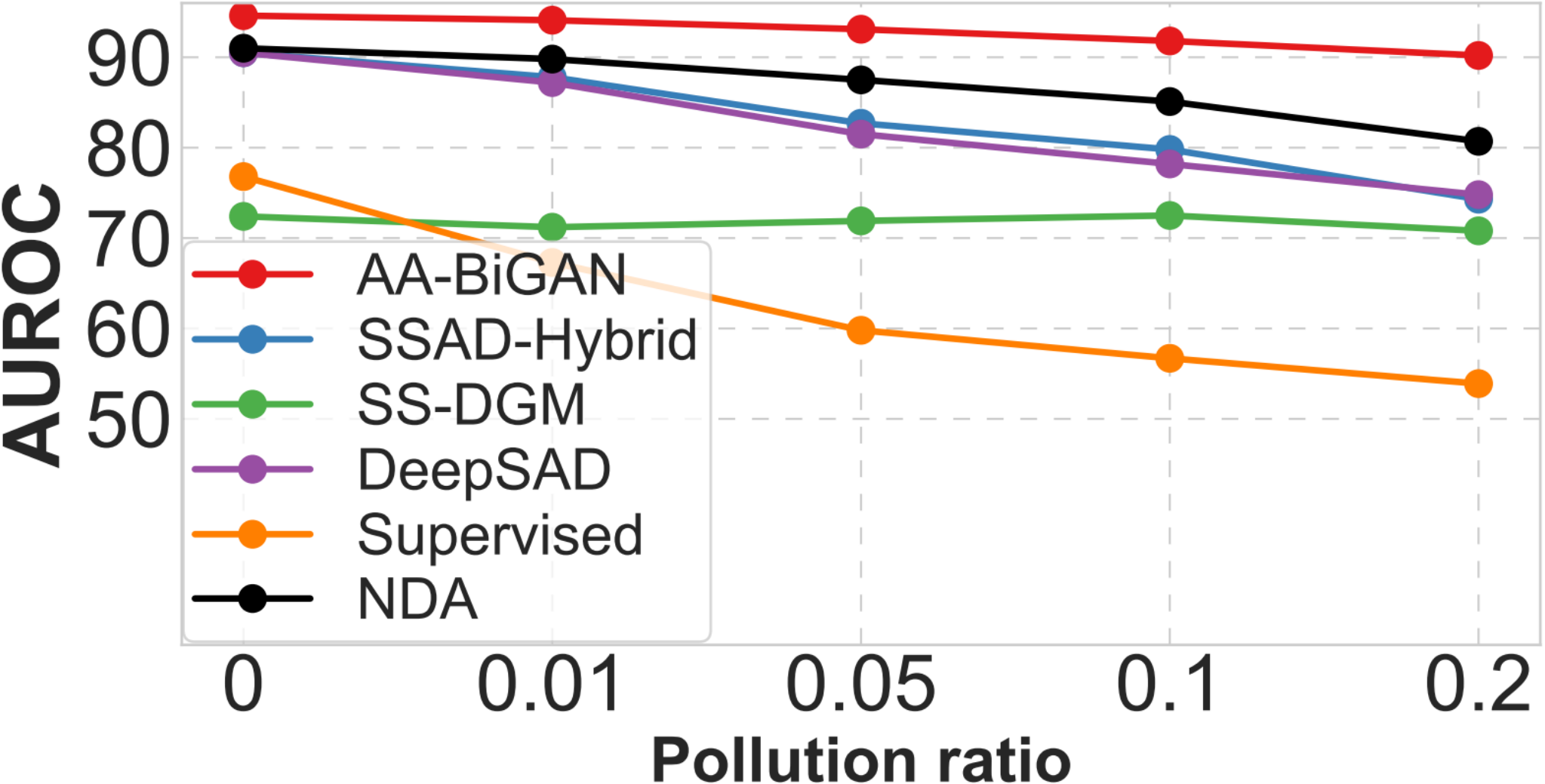}
        \vspace{-6mm}
        \caption*{\scriptsize F-MNIST}
    \end{minipage}
    \hspace{0.02\linewidth}
    \begin{minipage}{0.46\linewidth}
        \centering
        \includegraphics[width=1\linewidth]{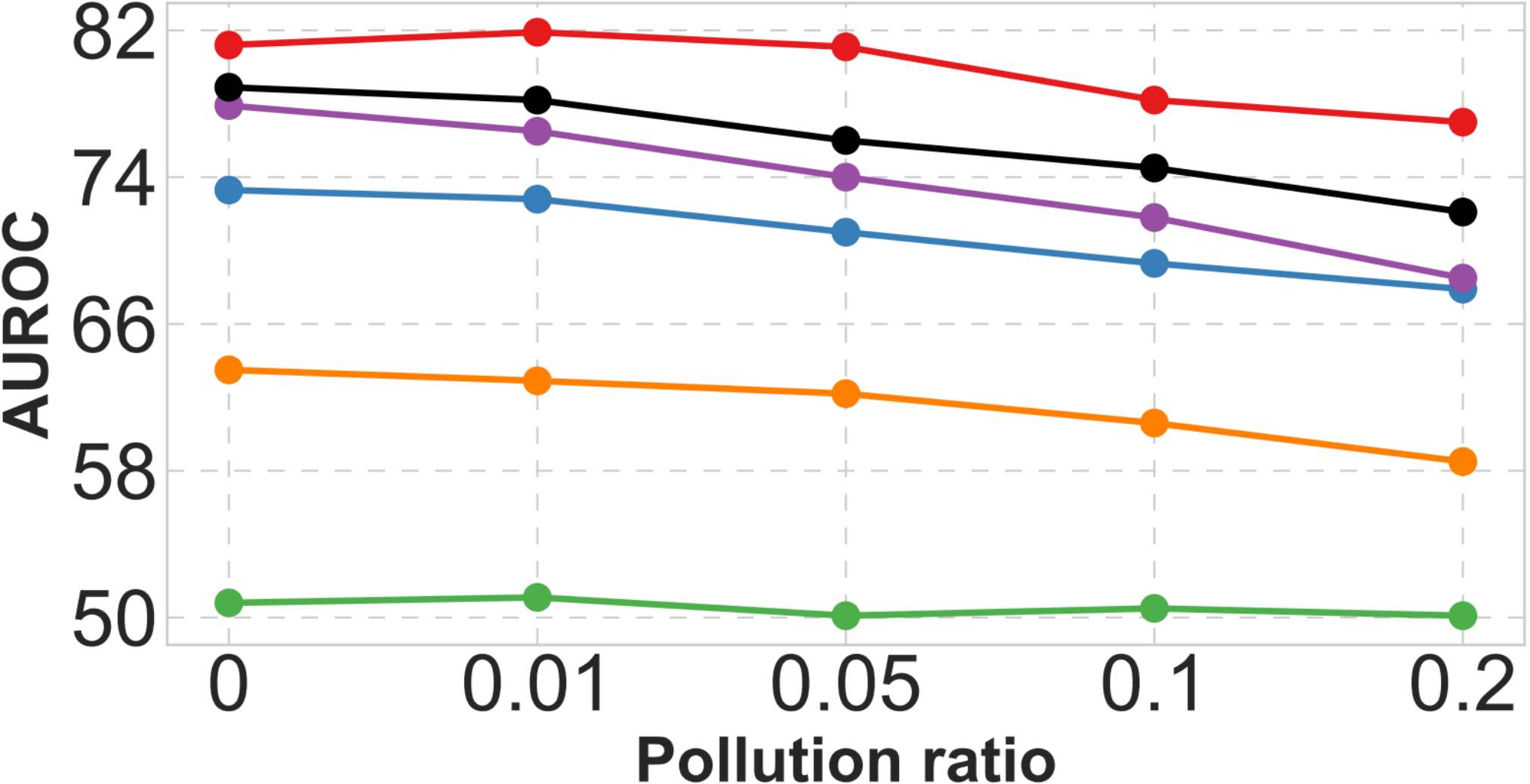}
        \vspace{-6mm}
        \caption*{\scriptsize CIFAR-10}
    \end{minipage}
    \vspace{-2mm}
\caption{The detection performance as a function of pollution ratio $\gamma_p$ on F-MNIST and CIFAR-10 datasets.}
\vspace{-2mm}
\label{img:resultp}
\end{figure}

\begin{figure}[!tb]
    \centering
    \begin{minipage}{0.46\linewidth}
        \centering
        \includegraphics[width=1\linewidth]{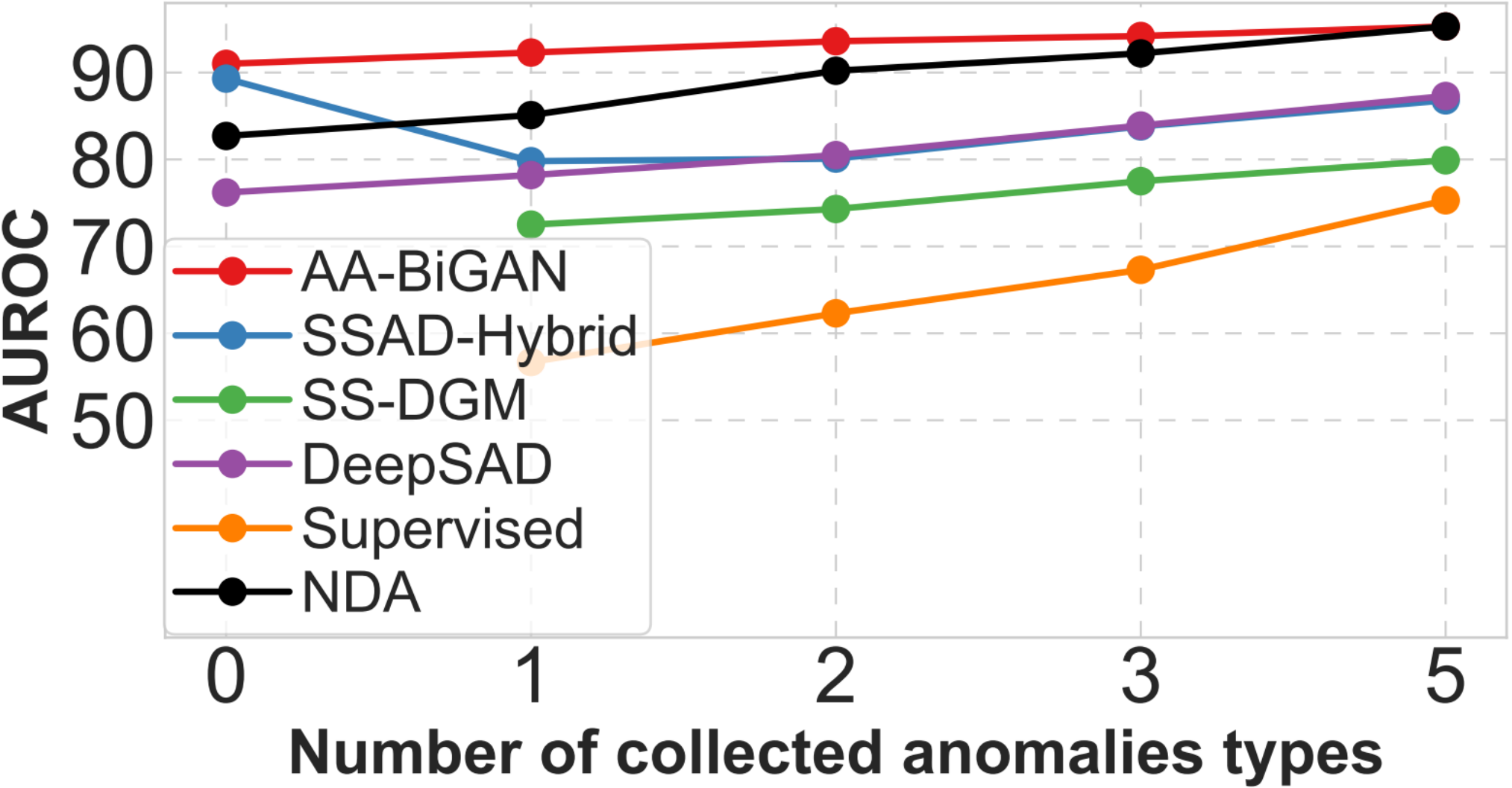}
        \vspace{-6mm}
        \caption*{\scriptsize F-MNIST}
    \end{minipage}
    \hspace{0.02\linewidth}
    \begin{minipage}{0.46\linewidth}
        \centering
        \includegraphics[width=1\linewidth]{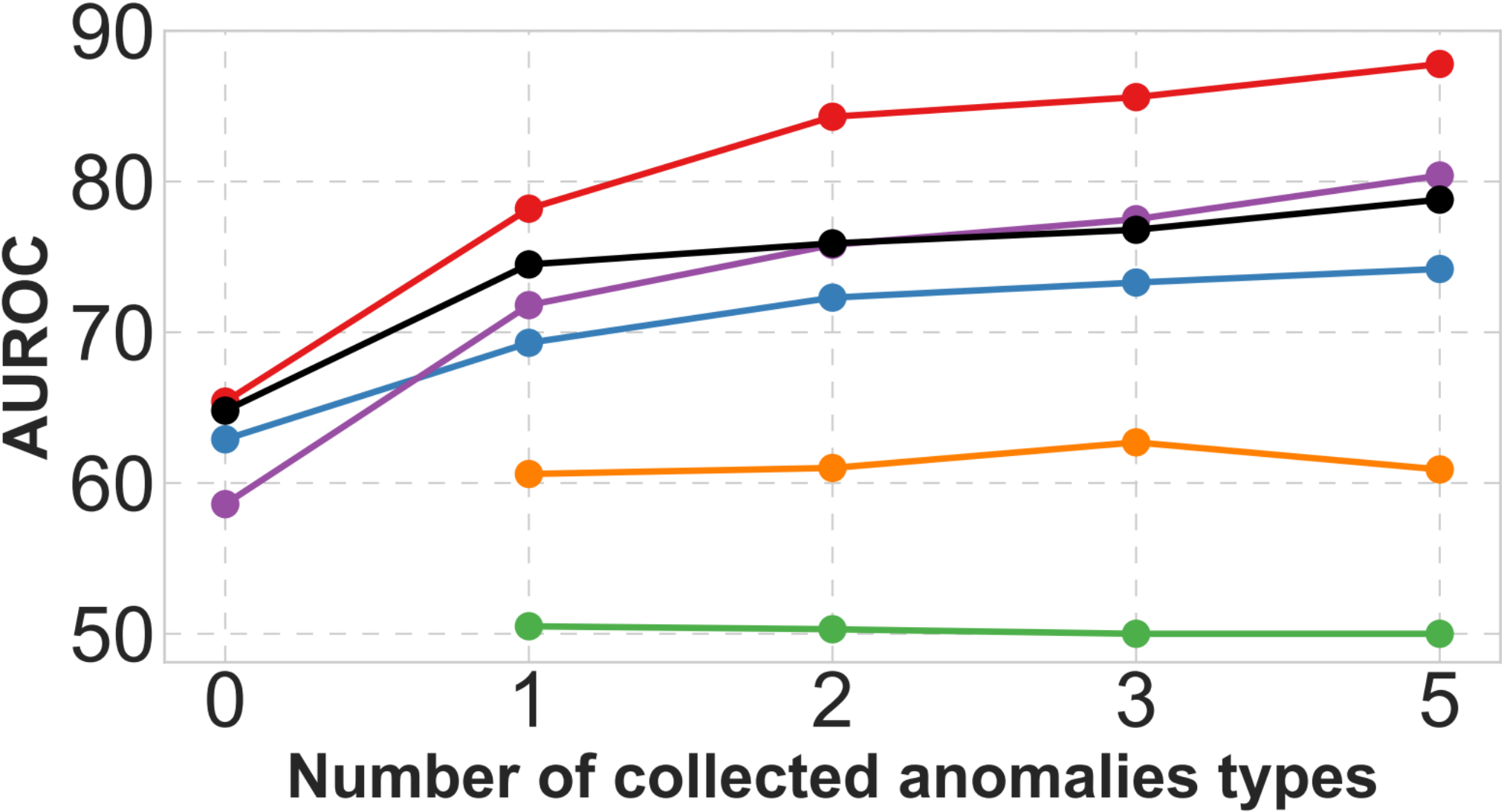}
        \vspace{-6mm}
        \caption*{\scriptsize CIFAR-10}
    \end{minipage}
    \vspace{-2mm}
\caption{The detection performance as a function of the number of categories of collected anomalies $k_l$.}
\vspace{-2.5mm}
\label{img:resultk}
\end{figure}

\paragraph{Impact of the diversity of collected anomalies}
Fig. \ref{img:resultk} shows how the performance of different methods varies as the number of categories of collected anomalies increases from $0$ to $5$, in which $\gamma_l$ and $\gamma_p$ are fixed as $0.05$ and $0.1$, respectively. It can be seen that the performance of all considered methods improves steadily as the number of categories increases from $0$ to $5$. This is consistent with our intuition since more types of anomalies are exposed to the models. However, the gain of our proposed method becomes less significant as more and more categories of anomalies are added into the training. This is easy to understand, when anomalies from five categories are used, it means anomalies from over a half of anomalous categories are accessible during the training, making the problem less challenging and hence the advantage of our method less obvious. But due to the extreme diversity of anomalies in real-world applications, the collected anomalies typically can only account for a very small fraction of all types, which suggests that the setups with small number of categories are actually more meaningful.

\begin{table}[!tb]
	
	\centering
	\resizebox{\columnwidth}{!}
	{\begin{tabular}{lc<{\centering}c<{\centering}c<{\centering}c<{\centering}c<{\centering}}
			\toprule
			\makecell[l]{\\Data} &  \makecell[c]{\\AAD} &  \makecell[c]{SSAD\\Hybrid} & \makecell[c]{Supervised\\ Classifier} & \makecell[c]{Deep\\ SAD}  & \makecell[c]{AA-BiGAN \\(Ours)}\\
			\midrule
			Arrhythmia & 75.8 $\!\pm\!$ 3.2 & 78.3 $\!\pm\!$ 5.1  &39.2 $\!\pm\!$ 9.5 & 75.9 $\!\pm\!$ 8.7 & \textbf{80.7 $\!\pm\!$ 3.2}\\
			
			Cardio & 90.7 $\!\pm\!$ 2.1 & 86.3 $\!\pm\!$ 5.8 & 83.2 $\!\pm\!$ 9.6 &95.0 $\!\pm\!$ 1.6  & \textbf{98.0 $\!\pm\!$ 1.2} \\
			Satellite & 77.2 $\!\pm\!$ 4.1 & 86.9 $\!\pm\!$ 2.8 & 87.2 $\!\pm\!$ 2.1 &\textbf{91.5 $\!\pm\!$ 1.1} & 87.4 $\!\pm\!$ 2.3 \\
			Satimage-2 &\textbf{99.9 $\!\pm\!$ 0.1} & 96.8 $\!\pm\!$ 2.1 & 99.1 $\!\pm\!$ 0.1 &\textbf{99.9 $\!\pm\!$ 0.1} & \textbf{99.9 $\!\pm\!$ 0.1}\\
			Shuttle & 99.0 $\!\pm\!$ 0.2 & 97.7 $\!\pm\!$ 1.0 & 95.1 $\!\pm\!$ 8.0 &98.4 $\!\pm\!$ 0.9 & \textbf{99.1 $\!\pm\!$ 0.1} \\
			Thyroid & 96.5 $\!\pm\!$ 0.8 & 95.3 $\!\pm\!$ 3.1 & 97.8 $\!\pm\!$ 2.6 & 98.6 $\!\pm\!$ 0.9 & \textbf{98.9 $\!\pm\!$ 0.1}\\
			
			\bottomrule
	\end{tabular}}
	\vspace{-2.5mm}
	\caption{AUROC on classic anomaly detection datasets.}
	\label{tal:4}
	\vspace{-3.5mm}
\end{table}

\paragraph{Performance on other anomaly detection datasets}
We evaluated our method on six other classic anomaly detection datasets. Table \ref{tal:4} shows the performance of our proposed model under the scenario of $\gamma_{l}=0.01$ and $\gamma_{p}=0$. From the table, it can be seen our proposed model overall outperforms current baseline methods. Even on Arrhythimia dataset, which contains less than $500$ samples, our model still achieve a $2\%$ performance improvement, demonstrating the competitiveness of the proposed method on small datasets.

\section{Conclusion}
In this paper, we studied the problem of anomaly detection under the circumstance that a handful of anomalies are available during the training. To effectively leverage the incomplete anomalous information to help anomaly detection, an anomaly-aware GAN is developed, which is able to explicitly avoid assigning probabilities for the collected anomalies, apart from the basic capabilities of modeling the distribution of normal samples. To facilitate the computation of anomaly detection criteria like reconstruction error, the anomaly-aware GAN is designed to be bidirectional. Extensive experiments demonstrated that under the circumstance of incomplete anomalous information, our model significantly outperformed existing baseline methods.

\section*{Acknowledgement}
This work is supported by the National Natural Science Foundation of China (No. 61806223, U1811264), Key R\&D Program of Guangdong Province (No. 2018B010107005), National Natural Science Foundation of Guangdong Province (No. 2021A1515012299). This work is also sponsored by CAAI-Huawei MindSpore Open Fund.


\bibliographystyle{named}
{\small\bibliography{ijcai22}}

\clearpage
\appendix
\twocolumn[
\begin{@twocolumnfalse}
\section*{\centering{\Large Supplementary Materials of Anomaly Detection by Leveraging Incomplete Anomalous Knowledge with Anomaly-Aware Bidirectional GANs}}
\vspace{15mm}
\end{@twocolumnfalse}
]

\theoremstyle{definition}
\setcounter{tocdepth}{4}
\setcounter{secnumdepth}{0}
\setcounter{theorem}{0} 

\section{A. Proof of Lemma \ref{lemma:1}}

\begin{lemma}
	If the discriminator $D$, generator $G$ and encoder $E$ are updated according to
	\begin{align} 
	\min_D V(D) & = {\mathbb{E}}_{({\mathbf{x}}, {\mathbf{z}})\sim p_E({\mathbf{x}}, {\mathbf{z}})}\left[(D({\mathbf{x}}, {\mathbf{z}}) - 1)^2\right] \nonumber \\
	&\quad +  {\mathbb{E}}_{({\mathbf{x}}, {\mathbf{z}}) \sim p_G({\mathbf{x}}, {\mathbf{z}})}\left[(D({\mathbf{x}}, {\mathbf{z}}) - 0)^2\right], 
	\tag{\ref{Discrim_BiLSGAN}}
    \end{align} 
    
	\begin{align} 
	\min_{G, E} V(G, E) & = {\mathbb{E}}_{({\mathbf{x}}, {\mathbf{z}})\sim p_E({\mathbf{x}}, {\mathbf{z}})}\left[\left(D({\mathbf{x}}, {\mathbf{z}}) - 0.5 \right)^2\right] \nonumber \\
	&\quad +  {\mathbb{E}}_{({\mathbf{x}}, {\mathbf{z}}) \sim p_G({\mathbf{x}}, {\mathbf{z}})}\!\! \left[\left(D({\mathbf{x}}, {\mathbf{z}}) - 0.5 \right)^2\right],
	\tag{\ref{GE_BiLSGAN}}
    \end{align}
	 the joint distributions $p_G({\mathbf{x}}, {\mathbf{z}}) = p_G({\mathbf{x}}|{\mathbf{z}})p({\mathbf{z}})$ and $p_E({\mathbf{x}}, {\mathbf{z}}) = p_E({\mathbf{z}}|{\mathbf{x}})p_{data}({\mathbf{x}})$ will converge to an identical distribution, that is, $p_G({\mathbf{x}}, {\mathbf{z}}) = p_E({\mathbf{x}}, {\mathbf{z}})$.
\label{lemma:1}
\end{lemma}
\begin{proof}
We first derive the optimal discriminator $D$ for a fixed $G,E$ as below:
\begin{align}
    D^{*}({\mathbf{x}},{\mathbf{z}})=\frac{p_{E}({\mathbf{x}},{\mathbf{z}})}{p_{E}({\mathbf{x}},{\mathbf{z}})+p_{G}({\mathbf{x}},{\mathbf{z}})}.
    \label{proof1:optimalD}
\end{align}
For the convince of writing, we write $p_{E}({\mathbf{x}},{\mathbf{z}})$, $p_{G}({\mathbf{x}},{\mathbf{z}})$ as $p_{E}$ and $p_{G}$.
We reformulate \eqref{GE_BiLSGAN} as follows:
\begin{align}
\begin{split}
    V(G,E) &=\!\iint_{\mathcal{X,Z}}\!\bigg \{\! \left(D^{*}\!-\!\frac{1}{2}\right)^{2}\!\! p_{E}
    \!+\!\left(D^{*}\!-\!\frac{1}{2}\right)^{2}\!\!p_{G}\! \bigg \} \mathrm{d} \boldsymbol{x}\mathrm{d} \boldsymbol{z}\\
    &=\!\iint_{\mathcal{X,Z}}\!\bigg \{\! \left(\frac{p_{E}}{p_{E}\!+\!p_{G}}\!-\!\frac{1}{2}\right)^{2}\!\! \left(p_{E}\!+\!p_{G}\right)\!\bigg \} \mathrm{d} \boldsymbol{x}\mathrm{d} \boldsymbol{z}\\
    &=\!\iint_{\mathcal{X,Z}}\bigg \{\! \left(\frac{\frac{1}{2}p_{E}\!-\!\frac{1}{2}p_{G}}{p_{E}\!+\!p_{G}}\right)^{2}\left(p_{E}\!+\!p_{G}\right)\!\bigg \} \mathrm{d} \boldsymbol{x}\mathrm{d} \boldsymbol{z}\\
    &=\!\iint_{\mathcal{X,Z}} \frac{\left(\frac{1}{2}p_{E}\!-\!\frac{1}{2}p_{G}\right)^2}{p_{E}\!+\!p_{G}} \mathrm{d} \boldsymbol{x}\mathrm{d} \boldsymbol{z}\\
    &=\frac{1}{4} \chi^{2}_{Pearson}\left(p_{E}\!+\!p_{G} || 2p_{G}\right).
    \label{proof1:GEobjective}
\end{split}
\end{align}
Where $\chi^{2}_{Pearson}$ is the Pearson $\chi^{2}$ divergence. Thus minimizing \eqref{proof1:GEobjective} yields minimizing the Pearson $\chi^{2}$ divergence between $p_{E}+p_{G}$ and $2p_{G}$. When $p_{G} = p_{E}$, \eqref{Discrim_BiLSGAN} and \eqref{GE_BiLSGAN} converge to the optimal solution. 
\end{proof}

\section{B. Proof of Lemma \ref{lemma_disjoint}}
\begin{theorem}
	Suppose $Supp(p_{\mathcal{X}}^+({\mathbf{x}})) \cap Supp(p_{\mathcal{X}}^-({\mathbf{x}}))=\emptyset$. 	If the discriminator $D$, generator $G$ and encoder $E$ are updated according to
	\begin{align}
	\min_D V(D) & = {\mathbb{E}}_{({\mathbf{x}}, {\mathbf{z}})\sim p_E^+({\mathbf{x}}, {\mathbf{z}})}\left[(D({\mathbf{x}}, {\mathbf{z}}) - 1)^2\right] \nonumber \\
	&\quad + {\mathbb{E}}_{({\mathbf{x}}, {\mathbf{z}}) \sim p_E^-({\mathbf{x}}, {\mathbf{z}})}\left[(D({\mathbf{x}}, {\mathbf{z}}) - 0)^2\right] \nonumber \\
	&\quad +  {\mathbb{E}}_{({\mathbf{x}}, {\mathbf{z}}) \sim p_G({\mathbf{x}}, {\mathbf{z}})}\left[(D({\mathbf{x}}, {\mathbf{z}}) - 0)^2\right],
	\tag{\ref{discrim_anoAwareGAN}}
    \end{align}
    
    \begin{align} 
	\min_{G, E} V(G, E) & = {\mathbb{E}}_{({\mathbf{x}}, {\mathbf{z}})\sim p_E^+({\mathbf{x}}, {\mathbf{z}})}\left[\left(D({\mathbf{x}}, {\mathbf{z}}) - 0.5 \right)^2\right] \nonumber \\
	&\quad + \! {\mathbb{E}}_{({\mathbf{x}}, {\mathbf{z}})\sim p_E^-({\mathbf{x}}, {\mathbf{z}})}\!\! \left[\left(D({\mathbf{x}}, {\mathbf{z}}) - 0.5 \right)^2\right] \nonumber \\
	&\quad +\!  {\mathbb{E}}_{({\mathbf{x}}, {\mathbf{z}}) \sim p_G({\mathbf{x}}, {\mathbf{z}})}\!\! \left[\left(D({\mathbf{x}}, {\mathbf{z}}) \!- \! 0.5 \right)^2\right],
	\tag{\ref{Gen_anoAwareGAN}}
    \end{align}
     after convergence, the two distributions $p_G({\mathbf{x}}, {\mathbf{z}})$ and $p_E^+({\mathbf{x}}, {\mathbf{z}})$ will be equal, that is, $p_G({\mathbf{x}}, {\mathbf{z}}) = p_E^+({\mathbf{x}}, {\mathbf{z}})$.
\label{lemma_disjoint}
\end{theorem}
\begin{proof}
For the convince of writing, we write  $p^{+}_{E}({\mathbf{x}},{\mathbf{z}})$, $p^{-}_{E}({\mathbf{x}},{\mathbf{z}})$, $p_{G}({\mathbf{x}},{\mathbf{z}})$ as $p^{+}_{E}$, $p^{-}_{E}$, and $p_{G}$. Expressing the expectations as integrals, \eqref{discrim_anoAwareGAN} and \eqref{Gen_anoAwareGAN} can be rewritten as:
\begin{align}
\begin{split}
    V(D)&=
    \iint_{\mathcal{X,Z}}\bigg \{\left(D({\mathbf{x}},{\mathbf{z}})-1\right)^{2} p_{E}^{+}
    \\&\quad +\left(D({\mathbf{x}},{\mathbf{z}})-0\right)^{2} p_{E}^{-}
    \\&\quad +\left(D({\mathbf{x}},{\mathbf{z}})-0\right)^{2} p_{G} \bigg \} \mathrm{d} \boldsymbol{x}\mathrm{d} \boldsymbol{z},
    \label{proof2:Dobjective}
\end{split}
\end{align}

\begin{align}
\begin{split}
    V(G,E)&=
    \iint_{\mathcal{X,Z}}\bigg \{ \left(D^{*}({\mathbf{x}},{\mathbf{z}})-0.5\right)^{2} p_{E}^{+}
    \\&\quad +\left(D^{*}({\mathbf{x}},{\mathbf{z}})-0.5\right)^{2} p_{E}^{-}
    \\&\quad +\left(D^{*}({\mathbf{x}},{\mathbf{z}})-0.5\right)^{2} p_{G} \bigg \} \mathrm{d} \boldsymbol{x}\mathrm{d} \boldsymbol{z},
    \label{proof2:GEobjective}
\end{split}
\end{align}
where $\mathcal{X} = Supp(p_{\mathcal{X}}^{+}) \cup Supp(p_{\mathcal{X}}^{-})$.

Forcing the first order derivative of the \eqref{proof2:Dobjective} to be 0, we obtain the optimal solution of discriminator:
\begin{align}
    D^*({\mathbf{x}}, {\mathbf{z}}) = \frac{p_E^+({\mathbf{x}}, {\mathbf{z}})}{p_E^+({\mathbf{x}}, {\mathbf{z}})+p_E^-({\mathbf{x}}, {\mathbf{z}})+p_G({\mathbf{x}}, {\mathbf{z}})}.
    \tag{\ref{optimaldiscrim_anoAwareGAN}}
    \label{proof2:Doptimal}
\end{align}
Using optimal solution of discriminator we can reformulate \eqref{proof2:GEobjective} as 
\begin{align}
\begin{split}
    V(G,E)&=\iint_{\mathcal{X,Z}}\bigg \{\left(\frac{p_{E}^{+}}{p_{E}^{+}+p_{E}^{-}+p_{G}}-\frac{1}{2}\right)^{2}p_{E}^{+}
    \\&\quad +\left(\frac{p_{E}^{+}}{p_{E}^{+}+p_{E}^{-}+p_{G}}-\frac{1}{2}\right)^{2} p_{E}^{-}
    \\&\quad +\left(\frac{p_{E}^{+}}{p_{E}^{+}+p_{E}^{-}+p_{G}}-\frac{1}{2}\right)^{2} p_{G} \bigg \} \mathrm{d} \boldsymbol{x}\mathrm{d} \boldsymbol{z}.
    \label{proof2:Geq1}
\end{split}
\end{align}
Because  $Supp(p_{\mathcal{X}}^{+}) \cap Supp(p_{\mathcal{X}}^{-})= \varnothing$, $p^{+}_{E}(x,z)=p^{+}_{\mathcal{X}}(x)p_{E}(z \vert x)$ and $p^{-}_{E}(x,z)=p^{-}_{\mathcal{X}}(x)p_{E}(z \vert x)$ are disjoint. Thus we divide entire  ${\mathcal{X,Z}}$ space into two sub-spaces, $(\mathcal{X,Z})_1=Supp(p_{E}^{+}(\mathcal{X,Z})) \cap Supp(p_{G}(\mathcal{X,Z}))$ and $(\mathcal{X,Z})_2=Supp(p_{E}^{-}(\mathcal{X,Z})) \cap Supp(p_{G}(\mathcal{X,Z}))$. On the subspace $(\mathcal{X,Z})_1$, $p_{E}^{+} \geq 0$ and $p_{E}^{-} = 0$. On the subspace $(\mathcal{X,Z})_1$, $p_{E}^{-} \geq 0$ and $p_{E}^{+} = 0$.
Using above properties to rewrite the \eqref{proof2:Geq1}, 
\begin{align}
\begin{split}
    V(G,E)&=\!\iint_{(\mathcal{X,Z})_{1}}\bigg \{ \left(\frac{p_{E}^{+}}{p_{E}^{+}\!+\!0\!+\!p_{G}}\!-\!\frac{1}{2}\right)^{2}\!\!p_{E}^{+}\!
    \\&\quad +\left(\frac{p_{E}^{+}}{p_{E}^{+}\!+\!0\!+\!p_{G}}\!-\!\frac{1}{2}\right)^{2}\!\!p_{G} \! \bigg \} \mathrm{d} \boldsymbol{x}\mathrm{d} \boldsymbol{z}\\
    \\&\quad +\!\iint_{(\mathcal{X,Z})_{2}}\bigg \{\! \left(\frac{0}{0\!+\!p_{E}^{-}\!+\!p_{G}}\!-\!\frac{1}{2}\right)^{2}\!\!p_{E}^{-}\!
    \\&\quad +\left(\frac{0}{0\!+\!p_{E}^{-}\!+\!p_{G}}\!-\!\frac{1}{2}\right)^{2}\!\! p_{G} \!\bigg \} \mathrm{d} \boldsymbol{x}\mathrm{d} \boldsymbol{z}.
    \label{proof2:Geq2}
\end{split}
\end{align}
Thus \eqref{proof2:Geq2} can be reformulated to
\begin{align}
\begin{split}
    &\!\!\!V(G,E)=\iint_{(\mathcal{X,Z})_{2}}\frac{1}{4}p_{G} \ \mathrm{d} \boldsymbol{x}\mathrm{d} \boldsymbol{z}+\!\iint_{\!(\mathcal{X,Z})_{2}}\frac{1}{4}p_{E}^{-}\ \mathrm{d} \boldsymbol{x}\mathrm{d} \boldsymbol{z}
    \\&\!\!\!\!\!+\!\!\iint_{(\mathcal{X,Z})_{1}}\!\!\!\bigg \{\!\!\!\left(\frac{p_{E}^{+}}{p_{E}^{+}\!+\!p_{G}}\!-\!\frac{1}{2}\right)^{2}\!\!\! p_{E}^{+}
    \!+\!\left(\frac{p_{E}^{+}}{p_{E}^{+}\!+\!p_{G}}\!-\!\frac{1}{2}\right)^{2}\!\!\! p_{G}\! \bigg \} \mathrm{d} \boldsymbol{x}\mathrm{d} \boldsymbol{z}.
    \label{proof2:Geq3}
\end{split}
\end{align}
Minimizing \eqref{proof2:Geq3},
\begin{align}
\begin{split}
    &\!\!\!V(G,E) 
    \\&\!\!\!\!\!\!= \iint_{(\mathcal{X,Z})_{2}}\frac{1}{4}p_{G} \ \mathrm{d} \boldsymbol{x}\mathrm{d} \boldsymbol{z}+\iint_{(\mathcal{X,Z})_{2}}\frac{1}{4}p_{E}^{-}\ \mathrm{d} \boldsymbol{x}\mathrm{d} \boldsymbol{z}
    \\&\quad \!\!\!\!\!+\iint_{(\mathcal{X,Z})_{1}}\bigg \{\left(p_{E}^{+} + p_{G}\right) \left(\frac{p_{E}^{+}}{p_{E}^{+}+p_{G}}-\frac{1}{2}\right)^{2} \bigg \}\mathrm{d} \boldsymbol{x}\mathrm{d} \boldsymbol{z} 
    \\&\!\!\!\!\!\! =\! \! \iint_{(\mathcal{X,Z})_{1}}\!\!\bigg \{\!\frac{\left(\frac{1}{2} p_{E}^{+}\!-\!\frac{1}{2} p_{G}\right)^{2}}{p_{E}^{+}\!+\!p_{G}}\!\bigg \}\mathrm{d} \boldsymbol{x}\mathrm{d} \boldsymbol{z}
    \!+\!\!\iint_{(\mathcal{X,Z})_{2}}\!\!\frac{1}{4}p_{G} \ \mathrm{d} \boldsymbol{x}\mathrm{d} \boldsymbol{z}\!+\!\frac{1}{4}\\
    \\&\!\!\!\!\!\! =\! \frac{1}{4} \chi^{2}_{Pearson}\!\left(p_{E}^{+}\!+\!p_{G}\ ||\ 2p_{G}\right)
    \!+\!\! \iint_{(\mathcal{X,Z})_{2}}\!\!\frac{1}{4}p_{G} \ \mathrm{d} \boldsymbol{x}\mathrm{d} \boldsymbol{z}\!+\!\frac{1}{4}.
    \label{proof2:Geq4}
\end{split}
\end{align}
Where $\chi^{2}_{Pearson}$is the Pearson  $\chi^{2}$ divergence, thus $\chi^{2}_{Pearson}\left(p_{E}^{+}+p_{G}||2p_{G}\right) \ge 0$, $p_{G}$ is the probaility distribution function, thus $\iint_{(\mathcal{X,Z})_{2}}\frac{1}{4}p_{G} \ \mathrm{d} \boldsymbol{x}\mathrm{d} \boldsymbol{z} \ge 0$, only when $p_{G} = p_{E}^{+}$, $\chi^{2}_{Pearson}\left(p_{E}^{+}+p_{G}||2p_{G}\right) = 0$ and $\iint_{(\mathcal{X,Z})_{2}}\frac{1}{4}p_{G} \ \mathrm{d} \boldsymbol{x}\mathrm{d} \boldsymbol{z} = 0$, thus minimizing \eqref{proof2:Geq4} can obtain $p_{G}(x,z)=p^{+}_{E}(x,z)$.

\end{proof}

\section{C. Proof of Lemma \ref{lemma_GenerousanoAwere}}
\begin{theorem}
	If the discriminator $D$, generator $G$ and encoder $E$ are updated according to
	\begin{align} 
		\min_D V(D) & = {\mathbb{E}}_{({\mathbf{x}}, {\mathbf{z}})\sim p_E^+({\mathbf{x}}, {\mathbf{z}})}\left[(D({\mathbf{x}}, {\mathbf{z}}) - a)^2\right] \nonumber \\
		&\quad + \! {\mathbb{E}}_{({\mathbf{x}}, {\mathbf{z}}) \sim p_E^-({\mathbf{x}}, {\mathbf{z}})}\!\! \left[\!\!\left(\!D({\mathbf{x}}, {\mathbf{z}}) \!-\! \frac{a\!+\!b}{2}\!\right)^2\!\right] \nonumber \\
		&\quad +  {\mathbb{E}}_{({\mathbf{x}}, {\mathbf{z}}) \sim p_G({\mathbf{x}}, {\mathbf{z}})}\left[(D({\mathbf{x}}, {\mathbf{z}}) - b)^2\right],
		\tag{\ref{discrim_GeneralanoAwareGAN}}
	\end{align}
	\begin{align} 
		\min_{G, E} V(G, E) & = {\mathbb{E}}_{({\mathbf{x}}, {\mathbf{z}})\sim p_E^+({\mathbf{x}}, {\mathbf{z}})}\left[\left(D({\mathbf{x}}, {\mathbf{z}}) - c \right)^2\right] \nonumber \\
		&\quad + \! {\mathbb{E}}_{({\mathbf{x}}, {\mathbf{z}})\sim p_E^-({\mathbf{x}}, {\mathbf{z}})}\!\! \left[\left(D({\mathbf{x}}, {\mathbf{z}}) - c \right)^2\right] \nonumber \\
		&\quad +\!  {\mathbb{E}}_{({\mathbf{x}}, {\mathbf{z}}) \sim p_G({\mathbf{x}}, {\mathbf{z}})}\!\! \left[\left(D({\mathbf{x}}, {\mathbf{z}}) \!- \! c \right)^2\right],
		\tag{\ref{Gen_GeneralanoAwareGAN}}
	\end{align}
the two distributions $p_G({\mathbf{x}}, {\mathbf{z}})$ and $p_E^+({\mathbf{x}}, {\mathbf{z}})$ after convergence will be equal, that is, $p_G({\mathbf{x}}, {\mathbf{z}}) = p_E^+({\mathbf{x}}, {\mathbf{z}})$.
\label{lemma_GenerousanoAwere}
\end{theorem}
\begin{proof}
For the convince of writing, we write  $p^{+}_{E}({\mathbf{x}},{\mathbf{z}})$, $p^{-}_{E}({\mathbf{x}},{\mathbf{z}})$, $p_{G}({\mathbf{x}},{\mathbf{z}})$ as $p^{+}_{E}$, $p^{-}_{E}$, and $p_{G}$. Expressing the expectations as integrals, \eqref{discrim_GeneralanoAwareGAN} and \eqref{Gen_GeneralanoAwareGAN} can be rewritten as:
\begin{align}
\begin{split}
    V(D)&=
    \iint_{\mathcal{X,Z}}\bigg \{\left(D({\mathbf{x}},{\mathbf{z}})-a\right)^{2} p_{E}^{+}
    \\& \quad +\left(D({\mathbf{x}},{\mathbf{z}})-\frac{a+b}{2}\right)^{2} p_{E}^{-}
    \\& \quad +\left(D({\mathbf{x}},{\mathbf{z}})-b\right)^{2} p_{G} \bigg \} \mathrm{d} \boldsymbol{x}\mathrm{d} \boldsymbol{z},
    \label{proof3:Dobjective}
\end{split}
\end{align}

\begin{align}
\begin{split}
    V(G,E)&=
    \iint_{\mathcal{X,Z}}\bigg \{ \left(D^{*}({\mathbf{x}},{\mathbf{z}})-c\right)^{2} p_{E}^{+}
    \\& \quad +\left(D^{*}({\mathbf{x}},{\mathbf{z}})-c\right)^{2} p_{E}^{-}
    \\& \quad +\left(D^{*}({\mathbf{x}},{\mathbf{z}})-c\right)^{2} p_{G} \bigg \} \mathrm{d} \boldsymbol{x}\mathrm{d} \boldsymbol{z},
    \label{proof3:GEobjective}
\end{split}
\end{align}
where $\mathcal{X} = Supp(p_{\mathcal{X}}^{+}) \cup Supp(p_{\mathcal{X}}^{-})$.

Forcing the first order derivative of the \eqref{proof3:Dobjective} to be 0, we obtain the optimal solution of discriminator:
\begin{align}
    D^*({\mathbf{x}}, {\mathbf{z}}) = \frac{ap_E^+({\mathbf{x}}, {\mathbf{z}}) +\frac{a+b}{2} p_E^-({\mathbf{x}}, {\mathbf{z}}) +b p_G ({\mathbf{x}}, {\mathbf{z}}) }{p_E^+({\mathbf{x}}, {\mathbf{z}})+p_E^-({\mathbf{x}}, {\mathbf{z}})+p_G({\mathbf{x}}, {\mathbf{z}})}.
    \label{proof3:Doptimal}
\end{align}
Using optimal solution of discriminator we can reformulate \eqref{proof3:GEobjective} as 
\begin{align}
\begin{split}
    &V(G,E)
    \\&\!\!=\iint_{\mathcal{X,Z}}\bigg \{\left(\frac{ap_E^+ \!+\!\frac{a+b}{2} p_E^- \!+\!b p_G }{p_E^{+}\!+\!p_E^{-}\!+\!p_G}\!-\!c\right)^{2}p_{E}^{+}\\
    &\!\! \quad +\!\left(\frac{ap_E^+ \!+\!\frac{a+b}{2} p_E^- \!+\!b p_G }{p_E^{+}\!+\!p_E^{-}\!+\!p_G}\!-\!c\right)^{2}\! p_{E}^{-}\\
    &\!\!\quad +\!\left(\frac{ap_E^+ \!+\!\frac{a+b}{2} p_E^- \!+\!b p_G }{p_E^{+}\!+\!p_E^{-}\!+\!p_G}\!-\!c\right)^{2}\! p_{G} \bigg \} \mathrm{d} \boldsymbol{x}\mathrm{d} \boldsymbol{z}\\
    &\!\!=\iint_{\mathcal{X,Z}}\!\!\left(\frac{ap_E^+ \!+\!\frac{a+b}{2} p_E^- \!+\!b p_G }{p_E^{+}\!+\!p_E^{-}\!+\!p_G}\!-\!c\right)^{2}\!\!\left(p_{E}^{+}\!+\!p_E^{-}\!+\!p_G\right) \mathrm{d} \boldsymbol{x}\mathrm{d} \boldsymbol{z}.
    \label{proof3:Geq1}
\end{split}
\end{align}
Introducing $\!\varphi(x) \!\!=\!\! \left(x-c\right)^2$, $\varphi(x)$ is a convex function, it satisfies $\varphi \!\left(\int_{-\infty}^{\infty} g(x)f(x) \mathrm{d} \boldsymbol{x}\right)\!\!\le\!\! \int_{-\infty}^{\infty} \varphi\left(g(x)\right)\!f(x)\ \mathrm{d} \boldsymbol{x} $, if $\int_{-\infty}^{\infty} f(x)  \mathrm{d} \boldsymbol{x}\!=\!1$, \eqref{proof3:Geq1}, it can be rewritten as 

\begin{align}
\begin{split}
    &V(G,E) 
    \\&\!\!=\!3\iint_{\mathcal{X,Z}}\! \frac{1}{3}\varphi\left(\!\frac{ap_E^+ \!+\!\frac{a+b}{2} p_E^- \!+\!b p_G \!}{p_E^{+}\!+\!p_E^{-}\!+\!p_G}\right)\!\!\left(p_{E}^{+}\!+\!p_{E}^{-}\!+\!p_{G}\right) \mathrm{d} \boldsymbol{x}\mathrm{d} \boldsymbol{z}.
    \label{proof3:Geq2}
\end{split}
\end{align}
Using the properties of convex functions, we can obtain 
\begin{align}
\begin{split}
    &V(G,E) 
    \\&\!\!=\!3\iint_{\mathcal{X,Z}}\! \frac{1}{3}\varphi\left(\!\frac{ap_E^+ \!+\!\frac{a+b}{2} p_E^- \!+\!b p_G }{p_E^{+}\!+\!p_E^{-}\!+\!p_G}\!\right)\!\left(p_{E}^{+}\!+\!p_{E}^{-}\!+\!p_{G}\right) \mathrm{d} \boldsymbol{x}\mathrm{d} \boldsymbol{z}\\
    &\!\!\ge\!3\ \varphi \bigg \{\! \iint_{\mathcal{X,Z}}\! \frac{1}{3}\!\left( p_{E}^{+}\!+\!p_{E}^{-}\!+\!p_{G}\right)\!\! \left(\!\!\frac{ap_E^+ \!+\!\frac{a+b}{2} p_E^- \!+\!b p_G }{p_E^{+}\!+\!p_E^{-}\!+\!p_G}\!\right)\! \mathrm{d} \boldsymbol{x}\mathrm{d} \boldsymbol{z}\! \bigg \}\\
    &\!\!= \!3\  \varphi\bigg \{\! \iint_{\mathcal{X,Z}}\! \frac{1}{3} \left(ap_E^+ \!+\!\frac{a+b}{2} p_E^- \!+\!b p_G \!\right)\!\mathrm{d} \boldsymbol{x}\mathrm{d} \boldsymbol{z}\! \bigg\}.
    \label{proof3:Geq3}
\end{split}
\end{align}
Because $p_{E}^{+},p_{E}^{-},p_{G}$ are all probability distribution, $\iint_{\mathcal{X,Z}} \frac{1}{3} \left(ap_E^+ +\frac{a+b}{2} p_E^- +b p_G \right) \mathrm{d} \boldsymbol{x}\mathrm{d} \boldsymbol{z} \!=\! \frac{a+b}{2}$, \eqref{proof3:Geq3} can obtain the following result

\begin{align}
\begin{split}
    &V(G,E) 
    \\&\ge\! 3\ \varphi\bigg \{ \!\iint_{\mathcal{X,Z}}\! \frac{1}{3}\! \left(ap_E^+ \!+\!\frac{a+b}{2} p_E^- \!+\!b p_G \right)\! \mathrm{d} \boldsymbol{x}\mathrm{d} \boldsymbol{z} \bigg\}\\
    &= 3\ \varphi\left(\frac{a+b}{2}\right).
    \label{proof3:Geq4}
\end{split}
\end{align}
Assuming that $p_{G} = p_{E}^{+}$, replacing $p_{G}$ in \eqref{proof3:Geq2} with $p_{E}^{+}$, we obtain 
\begin{align}
\begin{split}
    &V(G,E) 
    \\&\!\!=\!\iint_{\mathcal{X,Z}}\!\!\bigg \{\!\!\! \left(\frac{ap_E^+ \!+\!\frac{a+b}{2} p_E^- \!+\!b p_E^+ }{p_E^{+}\!+\!p_E^{-}\!+\!p_E^+}\!-\!c\right)^{2}\!\!\!\!\left(p_{E}^{+}\!+\!p_{E}^{-}\!+\!p_{E}^{+}\right)\!\!\!\bigg \} \mathrm{d} \boldsymbol{x}\mathrm{d} \boldsymbol{z}\\
    &\!\!=\! \iint_{\mathcal{X,Z}}\!\!\bigg \{\!\varphi\!\left(\frac{\left(a\!+\!b\right)p_{E}^{+}\!+\!\frac{a+b}{2}p_{E}^{-}}{2p_{E}^{+}\!+\!p_{E}^{-}}\right)\!\left(2p_{E}^{+}+p_{E}^{-}\right)\!\!\bigg \}\mathrm{d} \boldsymbol{x}\mathrm{d} \boldsymbol{z}\\
    & \!\!= \!\iint_{\mathcal{X,Z}}\!\!\bigg \{\!\varphi\!\left(\frac{a\!+\!b}{2}\right)\!\left(2p_{E}^{+}\!+\!p_{E}^{-}\right) \!\! \bigg \}\mathrm{d} \boldsymbol{x}\mathrm{d} \boldsymbol{z}\\
    & \!\!=\! 3\ \varphi\left(\frac{a\!+\!b}{2}\right).
    \label{proof3:Geq5}
\end{split}
\end{align}
Thus we observe that the minimum value of \eqref{proof3:Geq2}  is obtained when $p_{G} = p_{E}^{+}$.

\end{proof}

\begin{algorithm}[!htb]
    \textbf{Input}$\bm{:}$\\
    \hspace*{\algorithmicindent} Normal :${\mathcal{X^+}} \triangleq \{{\mathbf{x}}_1^+, {\mathbf{x}}_2^+, \cdots, {\mathbf{x}}_n^+\}$\\
    \hspace*{\algorithmicindent} Collected Anomalies :${\mathcal{X^-}} \triangleq \{{\mathbf{x}}_1^-, {\mathbf{x}}_2^-, \cdots, {\mathbf{x}}_m^-\}$\\
    \hspace*{\algorithmicindent} Adam learning rate: $\{\varepsilon_{DD^{'}},\varepsilon_{GE}\}$\\
    \textbf{Output}\\
    \hspace*{\algorithmicindent} Training model $\mathcal{W}^{*}$\\
    
    \caption{Optimization of Anomaly-Aware BiGAN \label{alg:1}}
    \begin{algorithmic}[1]
    \State \textbf{Initialize:}
    \Statex \hspace*{\algorithmicindent} Neural network weights: $\mathcal{W}$
    \For{each epoch}
        \For{each mini-batch}
            \State Draw mini-batch $\mathcal{M}$
            \State $J_{GE}=\widetilde V(G,E)$
            \State $J_{DD^{'}}=\widetilde V(D,D^{'})$
                  
           		\State $\mathcal{W}_{DD^{'}} \leftarrow \mathcal{W}_{DD^{'}} - \varepsilon_{DD^{'}} \cdot \nabla_{\mathcal{W}} J_{DD^{'}}(\mathcal{W}; \mathcal{M})$ 
           		\State $\mathcal{W}_{GE} \leftarrow \mathcal{W}_{GE} - \varepsilon_{GE} \cdot \nabla_{\mathcal{W}} J_{GE}(\mathcal{W}; \mathcal{M})$
           		
        \EndFor
    \EndFor

\end{algorithmic}
\end{algorithm}

\section{D. Reconstruction-Enhanced Objective Functions}
According to updating rules \eqref{discrim_GeneralanoAwareGAN}, \eqref{Gen_GeneralanoAwareGAN}, the joint distributions $p_G({\mathbf{x}}, {\mathbf{z}})$ and $p_E^+({\mathbf{x}}, {\mathbf{z}})$ after convergence will be equal. But in practical, a violation of cycle-consistency so $G(E(\mathbf{x})) \neq \mathbf{x}$ often occurs. This result in poor reconstruction ability. To overcome this issue, we introduce an extra discriminator $D^{'}$ to distinguish the pairs between $({\mathbf{x}}, {\mathbf{x}})$ and $({\mathbf{x}}, {\mathbf{\hat x}})$, where ${\mathbf{\hat x}}$ is the reconstruction of ${\mathbf{x}}$.    The objective functions are modified by adding the reconstruction-enhanced term,
\begin{align} \label{discrim_GeneralanoAwareGAN_XX}
		&\!\!\!\!\widetilde V(D,D^{'}) \nonumber
		\\&\!\!\!\! = V \left(D\right) + {\mathbb{E}}_{{\mathbf{x}}\sim p_E^+({\mathbf{x}})}\left[(D^{'}({\mathbf{x}}, {\mathbf{x}}) - a)^2\right] \nonumber \\
		& \!\!\!\!\quad + {\mathbb{E}}_{{\mathbf{x}} \sim p_E^-({\mathbf{x}})}\!\! \left[\!\left(\!D^{'}({\mathbf{x}}, {\mathbf{x}}) \!-\! \frac{a\!+\!b}{2}\!\right)^2\right] \nonumber \\
		&\!\!\!\! \quad +  {\mathbb{E}}_{{\mathbf{x}} \sim p_E^+({\mathbf{x}}),{\mathbf{z}} \sim p_E({\mathbf{z}}|{\mathbf{x}}),{\mathbf{\hat{x}}} \sim p_G({\mathbf{\hat{x}}}|{\mathbf{z}})}\left[\!(D^{'}({\mathbf{x}}, {\mathbf{\hat{x}}}) - b)^2\!\right],
	\end{align}
	\begin{align} \label{Gen_GeneralanoAwareGAN_XX}
		&\!\!\!\!\widetilde V(G,E) \nonumber
		\\&\!\!\!\!=V(G,E) + {\mathbb{E}}_{{\mathbf{x}}\sim p_E^+({\mathbf{x}})}\left[(D^{'}({\mathbf{x}}, {\mathbf{x}}) - c)^2\right] \nonumber \\
		& \!\!\!\!\quad +  {\mathbb{E}}_{{\mathbf{x}} \sim p_E^-({\mathbf{x}})} \left[\left(D^{'}({\mathbf{x}}, {\mathbf{x}}) \!-\! c\right)^2\right] \nonumber \\
		& \!\!\!\!\quad +  {\mathbb{E}}_{{\mathbf{x}} \sim p_E^+({\mathbf{x}}),{\mathbf{z}} \sim p_E({\mathbf{z}}|{\mathbf{x}}),{\mathbf{\hat{x}}} \sim p_G({\mathbf{\hat{x}}}|{\mathbf{z}})}\!\left[\!(D^{'}({\mathbf{x}}, {\mathbf{\hat{x}}}) - c)^2\!\right].
	\end{align}

The parameters can be updated by using the gradient of objective function. The complete training procedure is summarized in Algorithm \ref{alg:1}.

\section{E. Training Details}
For images from MNIST and F-MNIST, they are all resize to 32$\times$32$\times$3.  During training process we employ DCGAN architectures as Encoder-Generator framework. We implement our model on PyTorch and employ Adam optimizer for optimization. The learning rate of Generator and Encoder is set to be 0.0001, the learning rate of Discriminator is set to be 0.000025. The detailed information of the classic anomaly detection benchmarks are shown in Table \ref{Classic anomaly dataset}. We employ standard 3-layer MLP feed-forward architectures with 256-64-16 units.

\begin{table}[H]
\centering
\resizebox{\columnwidth}{!}
{\begin{tabular}{lrrr}
\toprule
Dataset & Numbers & Dimensions & \#outliers (\%) \\
\cmidrule(lr){1-1} \cmidrule(lr){2-2} \cmidrule(lr){3-3} \cmidrule(lr){4-4}
arrhythmia & 452 & 274 & 66 (14.6\%) \\
cardio & 1,831 & 21 & 176 (9.6\%) \\
satellite & 6,435 & 36 & 2,036 (31.6\%) \\
satimage-2 & 5,803 & 36 & 71 (1.2\%) \\
shuttle & 49,097 & 9 & 3,511 (7.2\%) \\
thyroid & 3,772 & 6 & 93 (2.5\%) \\
\bottomrule
\end{tabular}}
\caption{Classic anomaly detection benchmarks.}
\label{Classic anomaly dataset}
\end{table}

\section{F. Additional Experimental Results}

\begin{figure*}[!htb]
\centering
\subfigure[MNIST]{\label{fig:scatter_Mnist}\includegraphics[width=0.31\linewidth]{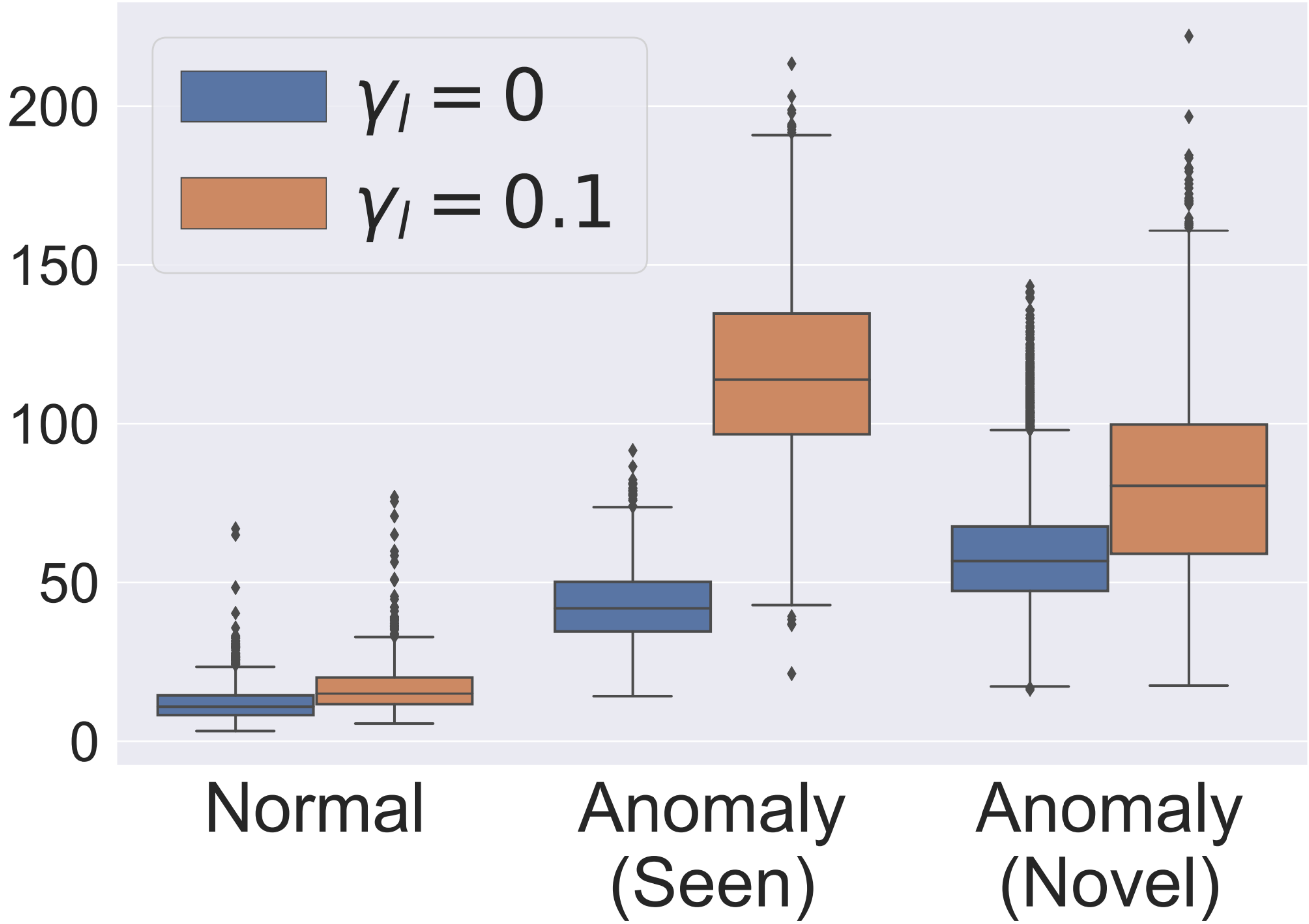}}
\hspace{0.02\linewidth}
\subfigure[F-MNIST]{\label{fig:scatter_Fmnist}\includegraphics[width=0.31\linewidth]{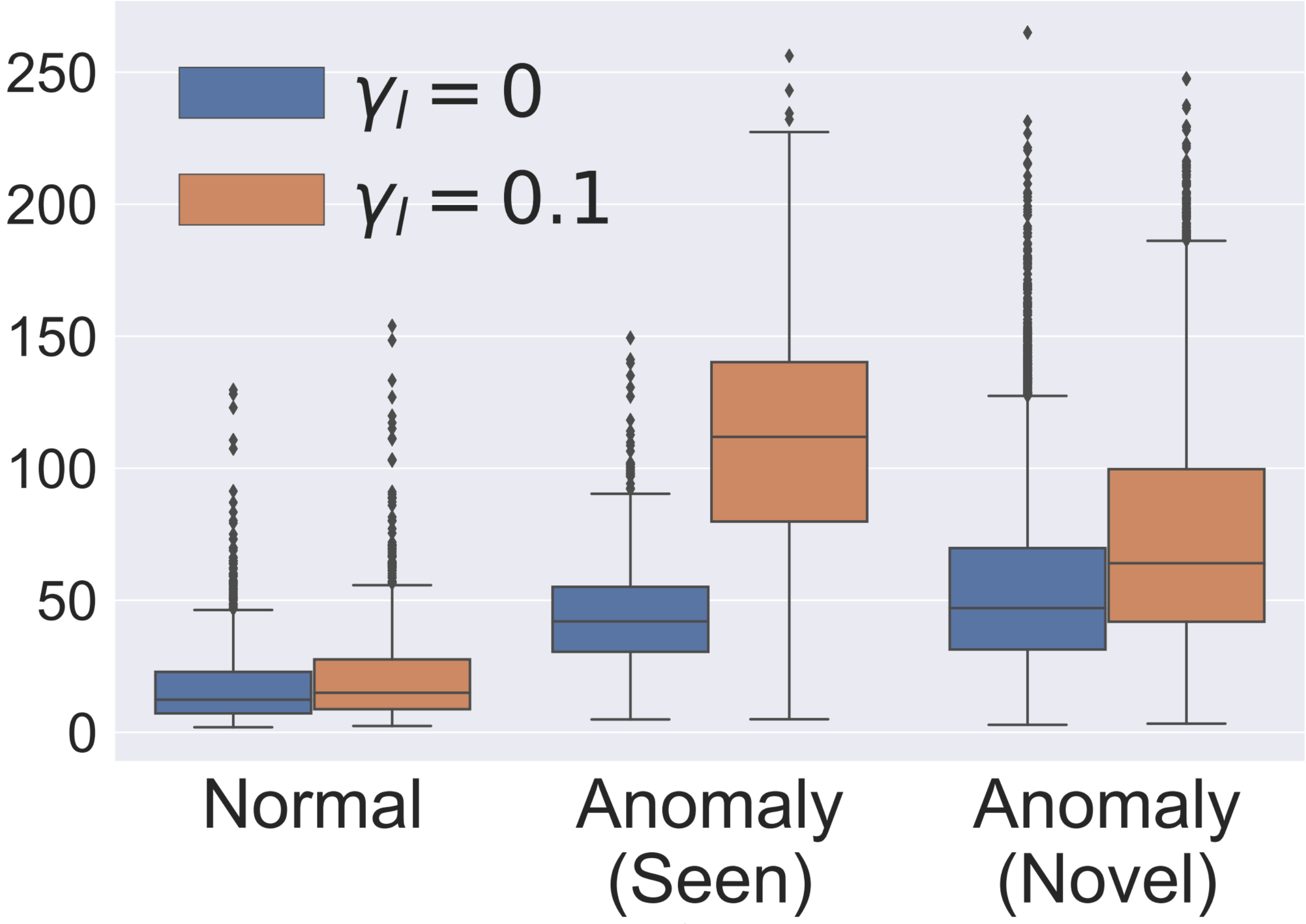}}
\hspace{0.02\linewidth}
\subfigure[CIFAR-10]{\label{fig:scatter_CIFAR10}\includegraphics[width=0.31\linewidth]{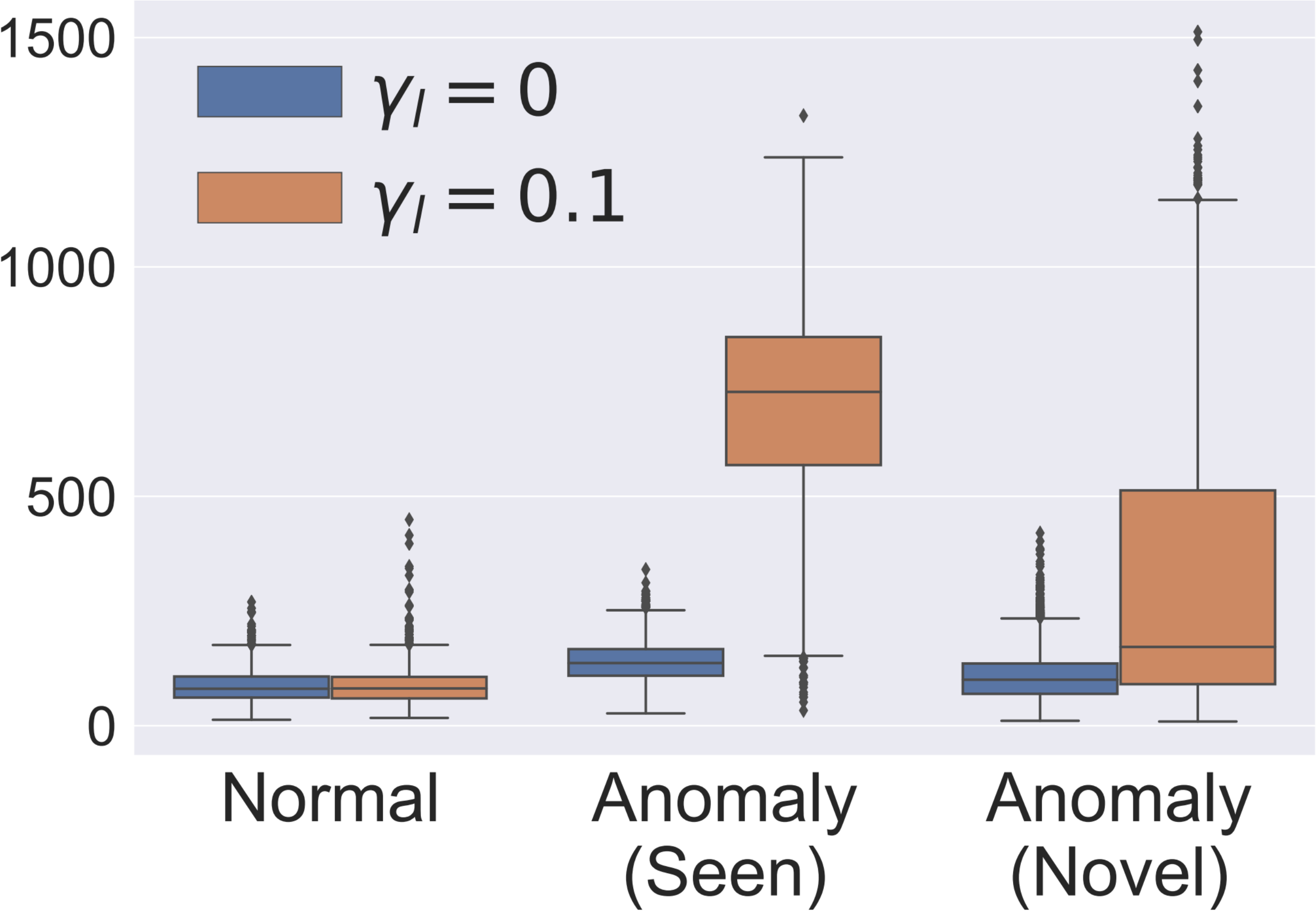}}
\caption{Boxplots of the reconstruction errors of normal samples, observed anomalies and novel anomalies, observed in our model with different collected-anomaly ratio $\gamma_l$.}
\label{img:recon}

\end{figure*}

\subsection{Reconstruction Error Visualization} 
In this experiment, we randomly choose one category as the normal-data category and another one as the available-anomaly category, from which a small proportion of samples are collected and used for training. Samples from the rest eight categories are viewed as novel anomalies. We train our proposed model under different collected-anomaly ratio $\gamma_l$ on the training dataset and then plot the boxplots of reconstruction errors of samples from the testing datasets in Fig. \ref{img:recon}, in which the case with $\gamma_l=0$ corresponds the scenario without using any anomalies during the training. It can be observed from the figure that as long as a proportion ($\gamma_l=0.1$) of anomalies are leveraged, the reconstruction error will be enlarged significantly on both of the observed and novel anomalies, while still keeping low on normal samples. The remarkable difference in reconstruction error will obviously facilitate the identification of anomalies and promote the detection accuracy. It can be also observed that the reconstruction error on the novel anomalies is not as large as that on the observed anomalies. This is reasonable since the model is trained to enlarge the reconstruction error on the category of observed anomalies. But thanks to the more accurate characterization to the data distribution, it is observed that the reconstruction errors on novel anomalies increase, too, making them easier to be detected.

\begin{figure}[!tb]
\centering
\includegraphics[width=0.98\linewidth]{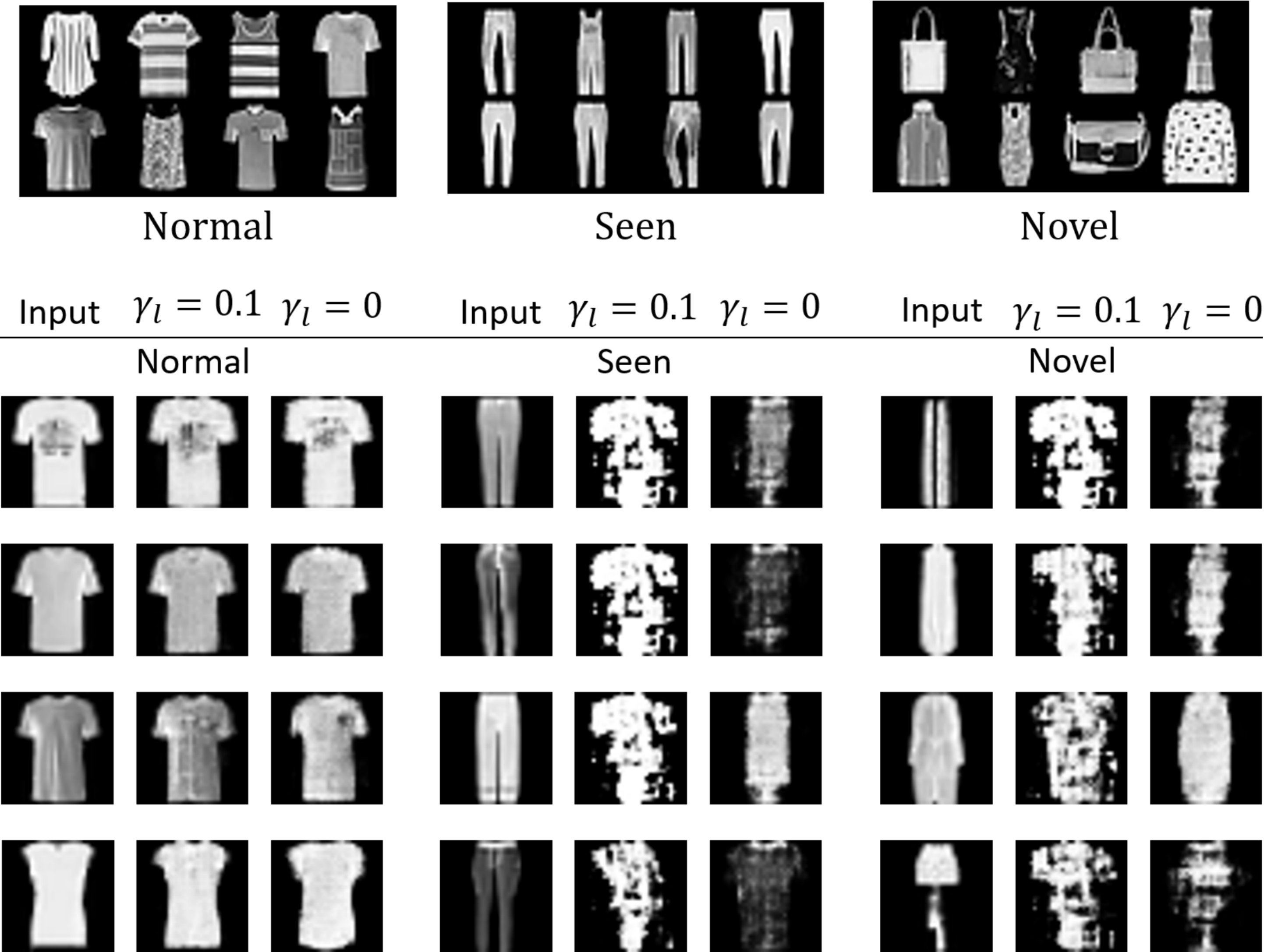}
\caption{Visualization analysis of reconstruction results on F-MNIST dataset. Top half shows the original images from normal, seen anomaly and novel anomaly category.  In the bottom half, three regions are correspond to the reconstruction results of normal, seen anomaly and novel anomaly respectively. The leftmost column of each region shows the input samples, the middle column shows the reconstructed images under $\gamma_l=0.1$ scenario, the rightmost column shows the reconstructed images under $\gamma_l=0$ scenario.}
\label{example}
\end{figure}

\vspace{-2mm}
\subsection{Reconstructed Images Visualization}
Considering that the image patterns in F-MNIST are easy to identify, we train our proposed model under different collected-anomaly ratio $\gamma_l$ on the F-MNIST dataset and then select the reconstructed images from the testing datasets to show in Fig. \ref{example}.  The case with $\gamma_l=0$ corresponds the scenario without using any anomalies during the training. It can be observed that for normal samples, whether utilizing the collected anomalies or not, 
the images can be well reconstructed. But without the leverage of collected anomalies, AA-BiGAN($\gamma_l=0$) seems to reconstruct the anomaly samples as well. The reconstruction  of the anomalous samples output by the AA-BiGAN($\gamma_l=0$) has a very similar shape contour to the original data. As long as a proportion($\gamma_l=0.1$)  of collected anomalies are leveraged, the reconstructed images of anomalies becomes blurry and noise, brings a significant difference between original images and reconstructed images. The comparison of 
reconstructed images intuitively explain how our model yields more discriminative reconstruction error criteria for better detection.




\subsection{Generation Ability Assessment}
Extensive experiments are conducted to verify our proposed model's generation quality. Fig. \ref{generated} shows our model can generate high quality images which patterns are easy to distinguish. To further evaluate the our model's generation ability, we conduct experiments referring to \cite{lucic2018gans} settings. In each experiment, we regard one category as anomalies and train on the other nine categories. We train $20$ epochs for F-MNIST and MNIST and $100$ epochs for CIFAR-10, with the best observed FID to be reported. Although the main propose of our model is to enhance the distinguishability between normal and anomaly samples rather than obtaining high quality generated images, it can still be observed that our model's FID score is overall in the same level as  existing generative models, showing the generation ability of our proposed model.

\begin{table}[!tb]
	
	\centering
	\normalsize
	{\begin{tabular}{lccccc}
			\toprule
			& \makecell[c]{MNIST } & \makecell[c]{F-MNIST } &  \makecell[c]{CIFAR-10}  \\
			\midrule
			GAN & 9.8 $\pm$ 0.9 & 29.6 $\pm$ 1.6 & 72.7 $\pm$ 3.6\\
			WGAN & 6.7 $\pm$ 0.4 & 21.5 $\pm$ 1.6 & 55.2 $\pm$ 2.3\\
			LSGAN & 7.8 $\pm$ 0.6 & 30.7 $\pm$ 2.2 & 87.1 $\pm$ 47.5\\
			VAE & 40.1 $\pm$ 0.7 & 100.9 $\pm$ 3.0 & 155.7 $\pm$ 11.6\\
			AA-BiGAN & 8.9 $\pm$ 1.0 & 27.2 $\pm$ 1.8 & 111.3 $\pm$ 10.1\\
			
			\bottomrule
	\end{tabular}}
	\caption{The FID comparison on three datasets with different models.}
	\label{tal:2}
\end{table}

\begin{figure}[!tb]
	\centering
	
	\subfigure[MNIST]{\includegraphics[width=0.29\linewidth]{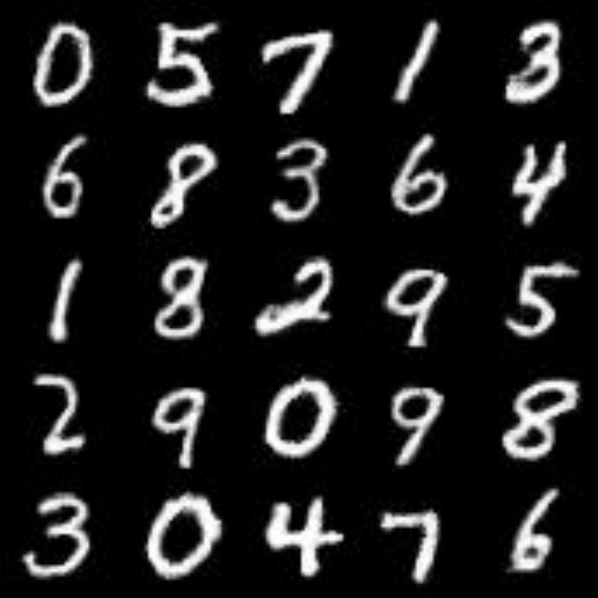}}
	\hspace{0.04\linewidth}
	\subfigure[F-MNIST]{\includegraphics[width=0.29\linewidth]{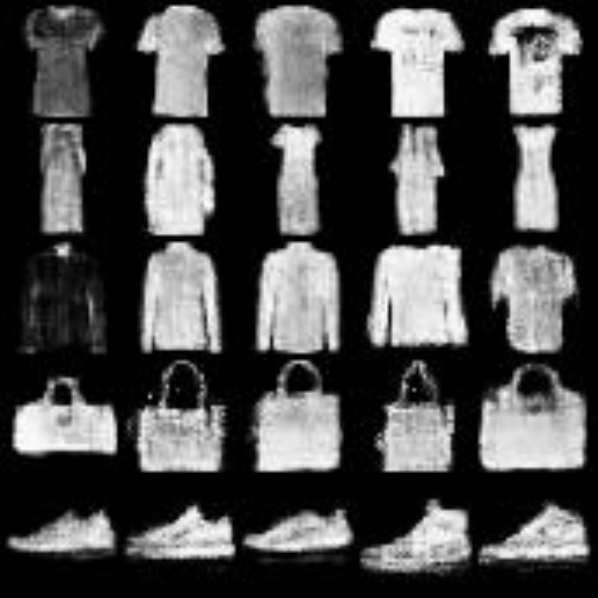}}
	\hspace{0.04\linewidth}
	\subfigure[CIFAR-10]{\includegraphics[width=0.29\linewidth]{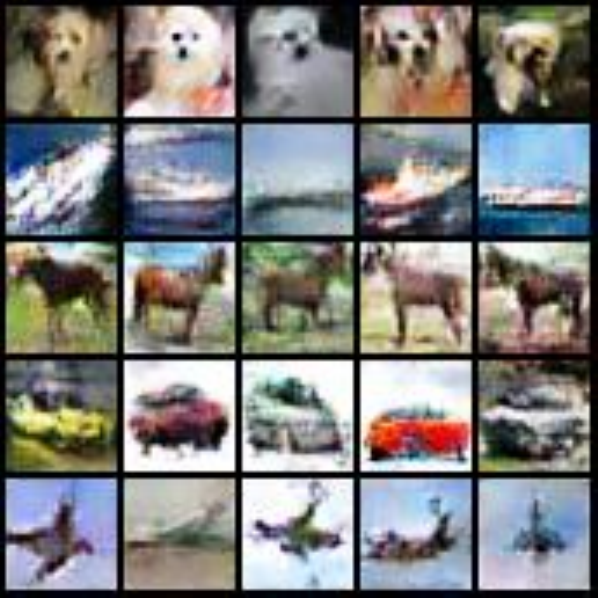}}
	\caption{Generated images on different datasets.}
	\label{generated}
\end{figure}

\subsection{Sensitivity Analysis}
We analyze the major hyperparameter $c$ of our proposed method. According to our developed theories, the parameter $c$ could be set arbitrarily,  in practical it better not make $c$ equal to $a$, $b$ and $\frac{a+b}{2}$ . We evaluate the performance of our model under four different values of $c$,  which is shown in Table \ref{Sensitivity analysis}. It can be seen that the performance is not sensitive to $c$. The results show that the best performance can be obtained when keeping the value of $c$ far away from the target value of other parameters. We thus set $c$ always to 3/4.

\begin{table}[!t]
	\setlength\tabcolsep{2.2pt} 
	\centering
	\resizebox{\columnwidth}{!}
	{\begin{tabular}{lc<{\centering}c<{\centering}c<{\centering}c<{\centering}c<{\centering}}
			\toprule
			\makecell[c]{$c$} & \makecell[c]{0.25} & \makecell[c]{0.4} & \makecell[c]{0.6} & \makecell[c]{0.75}\\
			\midrule
			$\gamma_l=0.05$, $\gamma_p=0$ & 94.5 $\!\pm\!$ 4.3 & 94.4 $\!\pm\!$ 4.3 & 94.5 $\!\pm\!$ 4.2 & 94.6 $\!\pm\!$ 4.2 \\
			
			$\gamma_l=0.05$, $\gamma_p=0.1$ & 91.6 $\!\pm\!$ 5.0 & 91.3 $\!\pm\!$ 5.3 & 91.3 $\!\pm\!$ 5.2 & 91.8 $\!\pm\!$ 4.9 \\
			\bottomrule
	\end{tabular}}
	\caption{Sensitivity analysis w.r.t. $c$ under different scenarios on F-MNIST dataset.}
	\label{Sensitivity analysis}
	
\end{table}

\begin{table*}[!htb]
\setlength\tabcolsep{3pt} 
\centering

\scriptsize
{
\begin{tabular}{ccccccccccccc}
\toprule
\makecell[c]{\\Data} & \makecell[c]{\\$\gamma_p$} & \makecell[c]{\\OCSVM} & \makecell[c]{Deep\\ SVDD} & \makecell[c]{\\ALAD} & \makecell[c]{\\CAE} &\makecell[c]{SSAD\\ Raw~} & \makecell[c]{SSAD\\ Hybrid~} & \makecell[c]{\\SS-DGM} & \makecell[c]{Deep\\ SAD} & \makecell[c]{\makecell[c]{Supervised\\ Classifier}} & \makecell[c]{\\NDA}& \makecell[c]{AA-BiGAN \\ (Ours)}\\
\cmidrule(lr){1-1} \cmidrule(lr){2-2} \cmidrule(lr){3-3} \cmidrule(lr){4-4} \cmidrule(lr){5-5} \cmidrule(lr){6-6} \cmidrule(lr){7-7} \cmidrule(lr){8-8} \cmidrule(lr){9-9} 
\cmidrule(lr){10-10} 
\cmidrule(lr){11-11} \cmidrule(lr){12-12} \cmidrule(lr){13-13}
& .00 & 96.3 $\pm$ 2.5 & 92.8 $\pm$ 4.9 & 96.2 $\pm$ 3.3 & 92.9 $\pm$ 5.7 & \textbf{97.9 $\pm$ 1.8} & 97.4 $\pm$ 2.0 & 92.2 $\pm$ 5.6 & 96.7 $\pm$ 2.4 & 94.5 $\pm$ 4.6 & 96.8 $\pm$ 3.2 & 97.8 $\pm$ 2.0 \\
& .01 & 95.6 $\pm$ 2.5 & 92.1 $\pm$ 5.1  & 93.3 $\pm$ 4.2  & 91.3 $\pm$ 6.1 & 96.6 $\pm$ 2.4 & 95.2 $\pm$ 2.3 & 92.0 $\pm$ 6.0 & 96.4 $\pm$ 2.7 & 91.5 $\pm$ 5.9 & 96.4 $\pm$ 2.7 & \textbf{96.9 $\pm$ 2.3}  \\
MNIST & .05 & 93.8 $\pm$ 3.9 & 89.4 $\pm$ 5.8 & 92.6 $\pm$ 5.2 & 87.2 $\pm$ 7.1 & 93.4 $\pm$ 3.4 & 89.5 $\pm$ 3.9 & 91.0 $\pm$ 6.9 & 93.5 $\pm$ 4.1 & 86.7 $\pm$ 7.4 & 93.2 $\pm$ 4.9 & \textbf{95.1 $\pm$ 3.0} \\
& .10 & 91.4 $\pm$ 5.1 & 86.5 $\pm$ 6.8 & 91.0 $\pm$ 5.6 & 83.7 $\pm$ 8.4 & 90.7 $\pm$ 4.4 & 86.0 $\pm$ 4.6 & 89.7 $\pm$ 7.5 & 91.2 $\pm$ 4.9 & 83.6 $\pm$ 8.2 & 90.4 $\pm$ 5.9 & \textbf{93.2 $\pm$ 4.3} \\
& .20 & 85.9 $\pm$ 7.6 & 81.5 $\pm$ 8.4 & 88.0 $\pm$ 6.2 & 78.6 $\pm$ 10.3 & 87.4 $\pm$ 5.6 & 82.1 $\pm$ 5.4 & 87.4 $\pm$ 8.6 & 86.6 $\pm$ 6.6 & 79.7 $\pm$ 9.4 & 87.3 $\pm$ 6.8 & \textbf{90.8 $\pm$ 5.1} \\
\cmidrule(lr){1-1} \cmidrule(lr){2-2} \cmidrule(lr){3-3} \cmidrule(lr){4-4} \cmidrule(lr){5-5} \cmidrule(lr){6-6} \cmidrule(lr){7-7} \cmidrule(lr){8-8} \cmidrule(lr){9-9} 
\cmidrule(lr){10-10}
\cmidrule(lr){11-11} 
\cmidrule(lr){12-12}  \cmidrule(lr){13-13}
& .00 & 91.2 $\pm$ 4.7 & 89.2 $\pm$ 6.2 & 91.1 $\pm$ 5.3 & 90.2 $\pm$ 5.8 & 94.0 $\pm$ 4.4 & 90.5 $\pm$ 5.9 & 71.4 $\pm$ 12.7 & 90.5 $\pm$ 6.5 & 76.8 $\pm$ 13.2 & 91.0 $\pm$ 7.0 &\textbf{94.6 $\pm$ 4.2}\\
& .01 & 91.5 $\pm$ 4.6 & 86.3 $\pm$ 6.3 & 91.3 $\pm$ 4.2  & 87.1 $\pm$ 7.3 & 92.2 $\pm$ 4.9 & 87.8 $\pm$ 6.1 & 71.2 $\pm$ 14.3 & 87.2 $\pm$ 7.1 & 67.3 $\pm$ 8.1 & 89.8 $\pm$ 8.6 & \textbf{94.1 $\pm$ 4.3}\\
F-MNIST & .05 & 90.7 $\pm$ 4.9 & 80.6 $\pm$ 7.1 & 90.5 $\pm$ 4.5 & 81.6 $\pm$ 9.6 & 88.3 $\pm$ 6.2 & 82.7 $\pm$ 7.8 & 71.9 $\pm$ 14.3 & 81.5 $\pm$ 8.5 & 59.8 $\pm$ 4.6 & 87.5 $\pm$ 10.5 & \textbf{93.1 $\pm$ 4.6}  \\
& .10 & 89.3 $\pm$ 6.2 & 76.2 $\pm$ 7.3 & 89.6 $\pm$ 4.5 & 77.4 $\pm$ 11.1 & 85.6 $\pm$ 7.0 & 79.8 $\pm$ 9.0 & 72.5 $\pm$ 15.5 & 78.2 $\pm$ 9.1 & 56.7 $\pm$ 4.1 &85.1 $\pm$ 12.9 & \textbf{91.8 $\pm$ 4.9} \\
& .20 & 88.1 $\pm$ 6.9 & 69.3 $\pm$ 6.3 & 88.3 $\pm$ 4.6 & 72.5 $\pm$ 12.6 & 81.9 $\pm$ 8.1 & 74.3 $\pm$ 10.6 & 70.8 $\pm$ 16.0 & 74.8 $\pm$ 9.4 & 53.9 $\pm$ 2.9 & 80.7 $\pm$ 15.1 & \textbf{90.2 $\pm$ 5.2} \\
\cmidrule(lr){1-1} \cmidrule(lr){2-2} \cmidrule(lr){3-3} \cmidrule(lr){4-4} \cmidrule(lr){5-5} \cmidrule(lr){6-6} \cmidrule(lr){7-7} \cmidrule(lr){8-8} \cmidrule(lr){9-9} 
\cmidrule(lr){10-10}
\cmidrule(lr){11-11} 
\cmidrule(lr){12-12} \cmidrule(lr){13-13}
& .00 & 63.8 $\pm$ 9.0 & 60.9 $\pm$ 9.4 & 65.0 $\pm$ 9.4 & 56.2 $\pm$ 13.2 &73.8 $\pm$ 7.6 & 73.3 $\pm$ 8.4 & 50.8 $\pm$ 4.7 & 77.9 $\pm$ 7.2 & 63.5 $\pm$ 8.0 & 78.9 $\pm$ 7.7 & \textbf{81.2 $\pm$ 6.4}\\
& .01 & 63.8 $\pm$ 9.3 & 60.5 $\pm$ 9.4  & 65.4 $\pm$ 9.8 & 56.2 $\pm$ 13.1 & 73.0 $\pm$ 8.0 & 72.8 $\pm$ 8.1 & 51.1 $\pm$ 4.7 & 76.5 $\pm$ 7.2 & 62.9 $\pm$ 7.3  & 78.2 $\pm$ 7.1 & \textbf{81.9 $\pm$ 6.6} \\
CIFAR-10 & .05 & 62.6 $\pm$ 9.2 & 59.6 $\pm$ 9.8 & 65.1 $\pm$ 9.3 & 55.7 $\pm$ 13.3 & 71.5 $\pm$ 8.2 & 71.0 $\pm$ 8.4 & 50.1 $\pm$ 2.9 & 74.0 $\pm$ 6.9 & 62.2 $\pm$ 8.2 & 76.0 $\pm$ 7.0 & \textbf{81.1 $\pm$ 6.7} \\
& .10 & 62.9 $\pm$ 8.2 & 58.6 $\pm$ 10.0 & 64.4 $\pm$ 9.3 & 55.4 $\pm$ 13.3 & 54.6 $\pm$ 13.3 & 69.3 $\pm$ 8.5 & 50.5 $\pm$ 3.6 & 71.8 $\pm$ 7.0 & 60.6 $\pm$ 8.3 & 74.5 $\pm$ 7.7& \textbf{78.2 $\pm$ 7.1} \\
& .20 & 61.9 $\pm$ 8.1 & 57.0 $\pm$ 10.6 & 64.2 $\pm$ 9.4 & 54.6 $\pm$ 14.3 & 67.8 $\pm$ 8.6 & 67.9 $\pm$ 8.1 & 50.1 $\pm$ 1.7 & 68.5 $\pm$ 7.1 & 58.5 $\pm$ 6.7 & 72.1 $\pm$ 8.5 & \textbf{77.0 $\pm$ 8.1} \\
\bottomrule
\end{tabular}
\caption{Complete detection performance as a function of pollution ratio $\gamma_p$ on MNIST, F-MNIST and CIFAR-10 datasets.}

\label{tablep}}

\end{table*}

\renewcommand \arraystretch{1.0}
\begin{table*}[!hbt]
\setlength\tabcolsep{3pt} 
\centering

\scriptsize
{\begin{tabular}{ccccccccccccc}
\toprule
\makecell[c]{\\Data} & \makecell[c]{\\$k_{l}$} & \makecell[c]{\\OCSVM} & \makecell[c]{Deep\\ SVDD} & \makecell[c]{\\ALAD} & \makecell[c]{\\CAE} & \makecell[c]{SSAD\\ Raw~} & \makecell[c]{SSAD\\ Hybrid~} & \makecell[c]{\\SS-DGM} & \makecell[c]{Deep\\ SAD} & \makecell[c]{Supervised\\ Classifier} & \makecell[c]{\\NDA}& \makecell[c]{AA-BiGAN \\ (Ours)}\\
\cmidrule(lr){1-1} \cmidrule(lr){2-2} \cmidrule(lr){3-3} \cmidrule(lr){4-4} \cmidrule(lr){5-5} \cmidrule(lr){6-6} \cmidrule(lr){7-7} \cmidrule(lr){8-8} \cmidrule(lr){9-9} 
\cmidrule(lr){10-10} 
\cmidrule(lr){11-11} \cmidrule(lr){12-12}\cmidrule(lr){13-13}
& 0 & \textbf{91.4 $\pm$ 5.1} & 86.5 $\pm$ 6.8 & 91.0 $\pm$ 5.6 & 83.7 $\pm$ 8.4 & 88.0 $\pm$ 6.0 & \textbf{91.4 $\pm$ 5.1} &  & 86.5 $\pm$ 6.8 &  & 88.5 $\pm$ 6.4 & 91.0 $\pm$ 5.6 \\
& 1 &  &  &  & & 90.7 $\pm$ 4.4 & 86.0 $\pm$ 4.6 & 89.7 $\pm$ 7.5 & 91.2 $\pm$ 4.9 & 83.6 $\pm$ 8.2 & 90.4 $\pm$ 3.9 &\textbf{94.5 $\pm$ 3.7} \\
MNIST & 2 &  &  &  & & 92.5 $\pm$ 3.6 & 87.7 $\pm$ 3.8 & 92.8 $\pm$ 5.3 & 92.0 $\pm$ 3.6 & 90.3 $\pm$ 4.6 & 94.7 $\pm$ 3.2 &\textbf{95.9 $\pm$ 2.7} \\
& 3 &  &  &  & & 93.9 $\pm$ 3.3 & 89.8 $\pm$ 3.3 & 94.9 $\pm$ 4.2 & 94.7 $\pm$ 2.8 & 93.9 $\pm$ 8.6 & 96.0 $\pm$ 3.0 &\textbf{96.3 $\pm$ 2.7} \\
& 5 &  &  &  & & 95.5 $\pm$ 2.5 & 91.9 $\pm$ 3.0 & 96.7 $\pm$ 2.3 & 97.3 $\pm$ 1.8 & 96.9 $\pm$ 1.7 & \textbf{97.9 $\pm$ 1.6} &  97.7 $\pm$ 1.5\\
\cmidrule(lr){1-1} \cmidrule(lr){2-2} \cmidrule(lr){3-3} \cmidrule(lr){4-4} \cmidrule(lr){5-5} \cmidrule(lr){6-6} \cmidrule(lr){7-7} \cmidrule(lr){8-8} \cmidrule(lr){9-9} 
\cmidrule(lr){10-10}
\cmidrule(lr){11-11} 
\cmidrule(lr){12-12}\cmidrule(lr){13-13} 
& 0 & 89.3 $\pm$ 6.2 & 76.2 $\pm$ 7.3 & 89.6 $\pm$ 4.5 & 77.4 $\pm$ 11.1 & 89.5 $\pm$ 6.1 & 89.3 $\pm$ 6.2 &  & 76.2 $\pm$ 7.3 & &73.3 $\pm$ 18.2 &\textbf{91.0 $\pm$ 4.5} \\
& 1 &  &  &  & & 85.6 $\pm$ 7.0 & 79.8 $\pm$ 9.0 & 72.5 $\pm$ 15.5 & 78.2 $\pm$ 9.1 & 56.7 $\pm$ 4.1 & 85.1 $\pm$ 12.9 &\textbf{92.3 $\pm$ 5.1} \\
F-MNIST & 2 &  &  &  &  & 87.8 $\pm$ 6.1 & 80.1 $\pm$ 10.5 & 74.3 $\pm$ 15.4 & 80.5 $\pm$ 8.2 & 62.3 $\pm$ 2.9 & 90.2 $\pm$ 7.9 &\textbf{93.6 $\pm$ 4.6} \\
& 3 &  &  &  &  & 89.4 $\pm$ 5.5 & 83.8 $\pm$ 9.4 & 77.5 $\pm$ 14.7 & 83.9 $\pm$ 7.4 & 67.3 $\pm$ 3.0 & 92.2 $\pm$ 6.7 & \textbf{94.2 $\pm$ 4.4} \\
& 5 &  &  &  & & 91.2 $\pm$ 4.8 & 86.8 $\pm$ 7.7 & 79.9 $\pm$ 13.8 & 87.3 $\pm$ 6.4 & 75.3 $\pm$ 2.7 & \textbf{95.3 $\pm$ 4.5} & \textbf{95.3 $\pm$ 4.4} \\
\cmidrule(lr){1-1} \cmidrule(lr){2-2} \cmidrule(lr){3-3} \cmidrule(lr){4-4} \cmidrule(lr){5-5} \cmidrule(lr){6-6} \cmidrule(lr){7-7} \cmidrule(lr){8-8} \cmidrule(lr){9-9} 
\cmidrule(lr){10-10}
\cmidrule(lr){11-11} 
\cmidrule(lr){12-12}\cmidrule(lr){13-13}
& 0 & 62.9 $\pm$ 8.2 & 58.6 $\pm$ 10.0 & 64.4 $\pm$ 9.3 & 55.4 $\pm$ 13.3 & 60.8 $\pm$ 10.7 & 62.9 $\pm$ 8.2 &  & 58.6 $\pm$ 10.0 & & 64.2 $\pm$ 6.4 & \textbf{65.4 $\pm$ 9.1}  \\
& 1 &  &   &  & & 69.8 $\pm$ 8.4 & 69.3 $\pm$ 8.5 & 50.5 $\pm$ 3.6 & 71.8 $\pm$ 7.0  & 60.6 $\pm$ 8.3 & 74.5 $\pm$ 7.7 & \textbf{78.2 $\pm$ 7.1}\\
CIFAR-10 & 2 &  &  & & & 73.0 $\pm$ 7.1 & 72.3 $\pm$ 7.5 & 50.3 $\pm$ 2.4 & 75.2 $\pm$ 6.4 & 61.0 $\pm$ 6.6 & 75.9 $\pm$ 8.1 & \textbf{84.3 $\pm$ 6.1}  \\
& 3 &  &  &  & & 73.8 $\pm$ 6.6 & 73.3 $\pm$ 7.0 & 50.0 $\pm$ 0.7 & 77.5 $\pm$ 5.9 & 62.7 $\pm$ 6.8 & 76.8 $\pm$ 7.6 & \textbf{85.6 $\pm$ 5.8} \\
& 5 &  &  &  & & 75.1 $\pm$ 5.5 & 74.2 $\pm$ 6.5 & 50.0 $\pm$ 1.0 & 80.4 $\pm$ 4.6 & 60.9 $\pm$ 4.6 & 78.8 $\pm$ 6.2& \textbf{87.8 $\pm$ 5.2} \\
\bottomrule
\end{tabular}
\caption{Complete detection performance as a function of the number of categories of collected anomalies $k_l$.}

\label{tablek}}
\end{table*}

\begin{table*}[!hbt]
	\scriptsize
	\centering
	\begin{tabular}{lc<{\centering}c<{\centering}c<{\centering}c<{\centering}c<{\centering}c<{\centering}c<{\centering}c<{\centering}c<{\centering}c<{\centering}}
			\toprule
			\makecell[l]{\\Data} & \makecell[c]{OCSVM\\ Hybrid~}&\makecell[c]{\\ CAE}&\makecell[c]{\\NDA}& \makecell[c]{Deep\\ SVDD}& \makecell[c]{\\AAD}& \makecell[c]{SSAD\\Hybrid~} & \makecell[c]{\\SS-DGM} & \makecell[c]{Supervised\\ Classifier} & \makecell[c]{Deep\\ SAD}  & \makecell[c]{AA-BiGAN \\(Ours)}\\
			\midrule
			Arrhythmia & 76.7 $\pm$ 6.2 & 74.0 $\pm$ 7.5 & 74.1 $\pm$ 8.9 & 74.6 $\pm$ 9.0  & 75.8 $\!\pm\!$ 3.2 & 78.3 $\pm$ 5.1 & 50.3 $\pm$ 9.8 &39.2 $\pm$ 9.5 & 75.9 $\pm$ 8.7 & \textbf{80.7 $\pm$ 3.2}\\
			
			Cardio & 82.8 $\pm$ 9.3 & 94.3 $\pm$ 2.0 & 86.0 $\pm$ 5.8 & 84.8 $\pm$ 3.6  & 90.7 $\!\pm\!$ 2.1 & 86.3 $\pm$ 5.8 & 66.2 $\pm$ 14.3 & 83.2 $\pm$ 9.6 &95.0 $\pm$ 1.6  & \textbf{98.0 $\pm$ 1.2} \\
			Satellite & 68.6 $\pm$ 4.8 & 80.0 $\pm$ 1.7 & 78.0 $\pm$ 3.3 & 79.8 $\pm$ 4.1& 77.2 $\!\pm\!$ 4.1 & 86.9 $\pm$ 2.8 & 57.4 $\pm$ 6.4 & 87.2 $\pm$ 2.1 &\textbf{91.5 $\pm$ 1.1} & 87.4 $\pm$ 2.3 \\
			Satimage-2 &96.7 $\pm$ 2.1 & 99.0$\pm$ 0.0 & 94.8 $\pm$ 2.3 & 98.3 $\pm$ 1.4 &\textbf{99.9 $\!\pm\!$ 0.1} & 96.8 $\pm$ 2.1 & 99.2 $\pm$ 0.6 & 99.1 $\pm$ 0.1 &\textbf{99.9 $\pm$ 0.1} & \textbf{99.9 $\pm$ 0.1}\\
			Shuttle &94.1 $\pm$ 9.5 & 98.2 $\pm$ 1.2 & 98.2 $\pm$ 0.5& 86.3 $\pm$ 7.5 &  99.0 $\!\pm\!$ 0.2 & 97.7 $\pm$ 1.0 & 97.9 $\pm$ 0.3 &95.1 $\pm$ 8.0 &98.4 $\pm$ 0.9 & \textbf{99.1 $\pm$ 0.1} \\
			Thyroid &91.2 $\pm$ 4.0 & 75.2 $\pm$ 10.2 & 91.9 $\pm$ 4.5& 72.0 $\pm$ 9.7 & 96.5 $\!\pm\!$ 0.8& 95.3 $\pm$ 3.1 & 72.7 $\pm$ 12.0 & 97.8 $\pm$ 2.6 & 98.6 $\pm$ 0.9 & \textbf{98.9 $\pm$ 0.1}\\
			\bottomrule
	\end{tabular}
	\caption{Complete detection performance on classic anomaly detection datasets.}
	\label{tabularre}
\end{table*}

\subsection{Empirical Study of Computational Efficiency} We investigate the training complexity by comparing the training duration of our method and ALAD \cite{ALAD18}. Theoretically, our model employs the same network architecture as ALAD model, the only difference is that we have added one more term to the updating function. Thus the complexity of our method is similar to the baseline ALAD model. On the F-MNIST and CIFAR-10 datasets, the one-epoch training time of our model is 8.833 seconds and 10.108 seconds, while that of the ALAD model is 7.924 seconds and 8.729 seconds. It can be seen that our model, though with much stronger performance, can be trained almost as efficiently as ALAD.

\subsection{Complete Tables of Experimental Results}
In addition to the semi-supervised models \cite{SSAD,SS-DGM,DeepSAD,NDA21,aad2020} mentioned in the main paper text, we also consider some unsupervised models \cite{OCSVM,DeepSVDD,ALAD18} for comparison. Table \ref{tablep} lists the complete experimental results as a function of pollution ratio $\gamma_p$ on the three image datasets, in which the ratio of available anomalies $\gamma_l$ is fixed as 0.05. Specifically, an averaged improvement of $1.3\%$, $1.6\%$, $3.7\%$ on MNIST, F-MNIST, CIFAR-10 are observed. As the pollution ratio $\gamma_p$ increases from 0.0 to 0.2, the superiority of our model is even more evident. Another noteworthy point is that ALAD model \cite{ALAD18} also appear to be more robust than distance-based models, which further confirms that the GAN-based probabilistic methods are often more tolerant of data pollution than the distance-based methods. 

Table \ref{tablek} lists the specific experimental results as a function of the number of categories of collected anomalies $k_l$ increases, in which $\gamma_l$ and $\gamma_p$ are fixed as 0.05 and 0.1. Our proposed model improves the average performance by $0.2\%$, $2.8\%$,$5.3\%$ on MNIST, F-MNIST, CIFAR-10. On the specific CIFAR-10 dataset, thanks to the strong ability of GANs in modeling the complex images, our model achieves a remarkable improvement in all level.

Table \ref{tabularre} lists the complete experimental results on tabular dataset. With the strong modeling capacity of deep neural networks, our proposed batch-processing model achieve an average improvement of $3.1\%$ performance compared to AAD. In comparison with the distance-based methods, the competitive results can be observed, which demonstrating the effectiveness of the proposed probabilistic method on classic anomaly detection datasets.  
\end{document}